\documentclass{article}
\usepackage{natbib}
\usepackage{amsmath,amsfonts,amsthm,amssymb}
\usepackage{xspace}
\usepackage{dsfont}
\usepackage{enumitem}

\usepackage{algorithmicx}
\usepackage{algorithm}
\usepackage{algpseudocode}

\newcommand{\bilinear}{{\sc Bilinear}\xspace}
\newcommand{\X}{\ensuremath{\mathcal{X}}\xspace}
\newcommand{\Y}{\ensuremath{\mathcal{Y}}\xspace}
\newcommand{\pdcoea}{PDCoEA\xspace}

\newcommand{\onemax}{\text{\sc OneMax}\xspace}

\usepackage{mathrsfs}
\newcommand{\filtuc}[1]        {\mathscr{F}_{#1}}
\newcommand{\filt}[1]        {\filtuc{#1}}

\newcommand{\prob}[1]{\Pr\left(#1\right)}

\newcommand{\expect}[1]{\mathbb{E}\left[#1\right]}

\newcommand{\expectt}[1]{\mathbb{E}_t\left[#1\right]}
\newcommand{\probt}[1]{\Pr{}_t\left(#1\right)}

\newtheorem{theorem}             {Theorem}
\newtheorem{lemma}      [theorem]{Lemma}
\newtheorem{corollary}  [theorem]{Corollary}

\newtheorem{definition} {Definition}

\newcommand{\indfe}[1]{\ensuremath{\mathds{1}_{#1}}} \newcommand{\indf}[1]{\indfe{\{#1\}}}

\DeclareMathOperator{\poly}{poly}
\DeclareMathOperator{\bin}{Bin}
\DeclareMathOperator{\unif}{Unif}

\newcommand{\Psel}[1]{\ensuremath{p_{\text{sel}}({#1})}}
\newcommand{\Punif}[1]{\ensuremath{p(#1)}}
\newcommand{\D}{\ensuremath{\mathcal{D}}}

\newcommand{\select}{\text{\tt select}\xspace}

\newcommand{\pmut}{\ensuremath{p_\mathrm{mut}}\xspace}

\newcommand{\WLOG}{w.l.o.g.\xspace}

\newcommand{\ab}{\hspace{0.125em}}                        \newcommand{\ie}{\hbox{i.\ab e.}\xspace}                  \newcommand{\eg}{\hbox{e.\ab g.}\xspace}

\usepackage{hyperref}
\usepackage{pdfpages}
\usepackage{bbm}
\usepackage{balance}

\title{Runtime
  Analysis of Competitive co-Evolutionary Algorithms for
  Maximin Optimisation of a Bilinear Function\footnote{This paper has been
    published as \cite{lehre_bilinear_algorithmica}, apart from the
    additional Theorem~\ref{thm:level-based-precise-lambda}.}}
\author{Per Kristian Lehre\\
  School of Computer Science\\
  University of Birmingham\\
  United Kingdom \\
  \texttt{p.k.lehre@bham.ac.uk}
}

\begin{document}
\maketitle

\begin{abstract}
Co-evolutionary algorithms have a wide range of applications, such
as in hardware design, evolution of strategies for board games, and
patching software bugs. However, these algorithms are poorly
understood and applications are often limited by pathological
behaviour, such as loss of gradient, relative over-generalisation,
and mediocre objective stasis. It is an open challenge to develop a
theory that can predict when co-evolutionary algorithms find
solutions efficiently and reliable.

This paper provides a first step in developing runtime analysis for
population-based competitive co-evolutionary algorithms. We provide
a mathematical framework for describing and reasoning about the
performance of co-evolutionary processes. To illustrate the framework,
we introduce a population-based co-evolutionary algorithm called \pdcoea, 
and prove that it obtains a solution to a bilinear maximin
optimisation problem in expected polynomial time. Finally,
we describe settings where \pdcoea needs
exponential time with overwhelmingly high probability to obtain a
solution.
\end{abstract}

\maketitle

\section{Introduction}

Many real-world optimisation problems feature a strategic aspect,
where the solution quality depends on the actions of other --
potentially adversarial -- players. There is a need for adversarial
optimisation algorithms that operate under realistic assumptions.
Departing from a traditional game theoretic setting, we assume two
classes of players, choosing strategies from ``strategy spaces''
$\mathcal{X}$ and $\mathcal{Y}$ respectively. The objectives of the
players are to maximise their individual ``payoffs'' as given by
payoff functions
$f,g:\mathcal{X}\times \mathcal{Y}\rightarrow \mathbb{R}$.

A fundamental algorithmic assumption is that there is insufficient
computational resources available to exhaustively explore the strategy
spaces $\mathcal{X}$ and $\mathcal{Y}$. In a typical real world
scenario, a strategy could consist of making $n$ binary
decisions. This leads to exponentially large and discrete strategy
spaces $\mathcal{X}=\mathcal{Y}=\{0,1\}^n$.  Furthermore, we can
assume that the players do not have access to or the capability to
understand the payoff functions. However, it is reasonable to assume
that players can make repeated queries to the payoff function
\cite{Fearnley2016}.  Together, these assumptions render many existing
approaches impractical, e.g., Lemke-Howson, best response dynamics,
mathematical programming, or gradient descent-ascent.

Co-evolutionary algorithms (CoEAs) (see \citep{popovici2012} for a survey)
could have a potential in adversarial optimisation, partly because
they make less strict assumptions than the classical methods. Two
populations are co-evolved (say one in $\X$, the other in $\Y$), where
individuals are selected for reproduction if they interact
successfully with individuals in the opposite population (e.g. as
determined by the payoff functions $f,g$). The hoped for outcome is
that an artificial ``arms race'' emerges between the populations,
leading to increasingly sophisticated solutions.  In fact, the
literature describe several successful applications, including design
of sorting networks \cite{hillis_co-evolving_1990}, software patching
\cite{arcuri_novel_2008}, and problems arising in cyber security
\cite{oreilly_adversarial_2020}.

It is common to separate co-evolution into co-operative and
competitive co-evolution. Co-operative co-evolution is attractive when
the problem domain allows a natural division into sub-components. For
example, the design of a robot can be separated into its morphology
and its control \cite{pollack_three_2001}.  A cooperative
co-evolutionary algorithm works by evolving separate ``species'',
where each species is responsible for optimising one sub-component of
the overall solution. To evaluate the fitness of a sub-component, it
is combined with sub-components from the other species to form a
complete solution. Ideally, there will be a selective pressure for the
species to cooperate, so that they together produce good overall
designs \cite{potter_cooperative_2000}.

The behaviour of CoEAs can be abstruse, where pathological population
behaviour such as loss of gradient, focusing on the wrong things, and
relativism \cite{watson_coevolutionary_2001} prevent effective
applications. It has been a long-standing open problem to develop a
theory that can explain and predict the performance of co-evolutionary
algorithms (see e.g. Section 4.2.2 in \cite{popovici2012}), notably
runtime analysis. Runtime analysis of EAs
\cite{RuntimeAnalysis2020Book} has provided mathematically rigorous
statements about the runtime distribution of evolutionary algorithms,
notably how the distribution depends on characteristics of the fitness
landscape and the parameter settings of the algorithm. Following from
the publication of the conference version of this paper, several other
results on the runtime of competitive co-evolutionary algorithms have
appeared considering variants of the \bilinear game introduced in
Section \ref{sec:probl-maxim-optim}. Hevia, Lehre and Lin analysed the runtime of Randomised
Local Search CoEA (RLS-PD) on \bilinear \cite{RLSPD}. Hevia and Lehre
analysed the runtime of $(1,\lambda)$ CoEA on a lattice variant of
\bilinear \cite{LatticeBilinear}.

The only rigorous runtime analysis of co-evolution the author is aware
of focuses on co-operative co-evolution. In a pioneer study, Jansen
and Wiegand considered the common assumption that co-operative
co-evolution allows a speedup for \emph{separable} problems
\cite{jansen_cooperative_2004}. They compared rigorously the runtime
of the co-operative co-evolutionary (1+1) Evolutionary Algorithm (CC
(1+1) EA) with the classical (1+1) EA.
Both algorithms follow the same template: They keep the single best
solution seen so far, and iteratively produce new candidate solution
by ``mutating'' the best solution. However, the algorithms use
different mutation operators. The CC (1+1) EA restricts mutation to
the bit-positions within one out of $k$ blocks in each iteration. The
choice of the current block alternates deterministically in each
iteration, such that in $k$ iterations, every block has been active
once.
The main conclusion from their analysis is that problem separability
is not a sufficient criterion to determine whether the CC (1+1) EA
performs better than the (1+1) EA. In particular, there are separable
problems where the (1+1) EA outperforms the CC (1+1) EA, and there are
inseparable problems where the converse holds. What the authors find
is that CC (1+1) EA is advantageous when the problem separability
matches the partitioning in the algorithm, and there is a benefit from
increased mutation rates allowed by the CC (1+1) EA.

Much work remains to develop runtime analysis of co-evolution.
Co-operative co-evolution can be seen as a particular approach to
traditional optimisation, where the goal is to maximise a given
objective function. In contrast, competitive co-evolutionary
algorithms are employed for a wide range of solution concepts
\cite{Ficici2004PhD}. It is unclear to what degree results about
co-operative CoEAs can provide insights about competitive
CoEAs. Finally, the existing runtime analysis considers the CC (1+1)
EA which does not have a population. However, it is particularly
important to study co-evolutionary population dynamics to understand
the pathologies of existing CoEAs.

This paper makes the following contributions:
Section~\ref{sec:co-evolutionary-framework} introduces a generic
mathematical framework to describe a large class of co-evolutionary
processes and defines a notion of ``runtime'' in the
context of generic co-evolutionary processes. We then discuss how the population-dynamics of these
processes can be described by a stochastic process.
Section~\ref{sec:level-based-theorem} presents an
analytical tool (a co-evolutionary level-based theorem) which can be
used to derive upper bounds on the expected runtime of co-evolutionary
algorithms.
Section~\ref{sec:probl-maxim-optim} specialises the problem setting to
maximin-optimisation, and introduces a theoretical benchmark problem
\bilinear. Section~\ref{sec:co-evol-algor}
introduces the algorithm \pdcoea which is a particular co-evolutionary
process tailored to maximin-optimisation. We then analyse the runtime
of \pdcoea on \bilinear using the level-based theorem, showing that
there are settings where the algorithm obtains a solution in
polynomial time. Since the publication of the conference version of
this paper, the \pdcoea has been applied to a cyber-security domain \cite{PDCoEADefendIT}.
In Section~\ref{sec:co-evol-error}, we
demonstrate that the \pdcoea possesses an ``error threshold'', i.e., a
mutation rate above which the runtime is exponential for any problem.
Finally, the appendix contains some technical results which have been
relocated from the main text to increase readability.

\subsection{Preliminaries}

For any natural number $n\in\mathbb{N}$, we define
$[n]:=\{1,2,\ldots, n\}$ and $[0..n]:=\{0\}\cup [n]$. For a filtration
$(\filt{t})_{t\in\mathbb{N}}$ and a random variable $X$ we use the
shorthand notation $\expectt{X} := \expect{X\mid\filt{t}}$. A random
variable $X$ is said to \emph{stochastically dominate} a random variable $Y$,
denoted $X\succeq Y$,
if and only if $\prob{Y\leq z}\geq \prob{X\leq z}$ for all $z\in\mathbb{R}$.
The
Hamming distance between two bitstrings $x$ and $y$ is denoted $H(x,y)$.
For any bitstring $z\in\{0,1\}^n$, $\|z\|:=\sum_{i=1}^n z_i$,
denotes the number of 1-bits in $z$.

\section{Co-Evolutionary Algorithms}\label{sec:co-evolutionary-framework}

This section describes in mathematical terms a broad class of
co-evolutionary processes (Algorithm~\ref{algo:coea-process}),
along with a definition of their runtime for a given solution concept.
The definition takes inspiration from level-processes (see Algorithm~1 in
\cite{levelbasedanalysis2018}) used to describe non-elitist
evolutionary algorithms.

\begin{algorithm}
  \caption{Co-evolutionary Process}
  \begin{algorithmic}[1]
    \Require Population size $\lambda\in\mathbb{N}$ and strategy spaces $\mathcal{X}$ and $\mathcal{Y}$.
    \Require Initial populations $P_0\in\mathcal{X}^\lambda$
    and $Q_0\in\mathcal{Y}^\lambda$.
    \For{each generation number $t\in\mathbb{N}_0$}
    \For{each interaction number $i\in[\lambda]$}\label{alg:coev:evaluate1}
    \State Sample an interaction $(x,y)\sim \mathcal{D}(P_t,Q_t).$ \label{alg:coev:evaluate2}
    \State Set $P_{t+1}(i) := x$ and $Q_{t+1}(i) := y$.
    \EndFor\label{alg:coev:evaluate3}
    \EndFor
    \end{algorithmic}
  \label{algo:coea-process}
\end{algorithm}

We assume that in each generation, the algorithm has two\footnote{The
  framework can be generalised to more populations.}  populations
$P\in\X^\lambda$ and $Q\in\Y^\lambda$ which we sometimes will refer to
as the ``predators'' and the ``prey''. Note that these terms are
adopted only to connect the algorithm with their biological
inspiration without imposing further conditions. In particular, we do
not assume that predators or prey have particular roles, such as one
population taking an active role and the other population taking a
passive role. We posit that in each generation, the populations
interact $\lambda$ times, where each interaction produces in a
stochastic fashion one new predator $x\in\X$ and one new prey
$y\in\Y$. The interaction is modelled as a probability distribution
$\mathcal{D}(P,Q)$ over $\X\times \Y$ that depends on the current
populations. For a given instance of the framework, the operator
$\mathcal{D}$ encapsulates all aspects that take place in producing
new offspring, such as pairing of individuals, selection, mutation,
crossover, etc. (See Section~\ref{sec:co-evol-algor} for a particular
instance of $\mathcal{D}$).

As is customary in the theory of evolutionary computation, the definition
of the algorithm does not state any termination criterion. The
justification for this omission is that the choice of termination
criterion does not impact the definition of runtime we will use.

Notice that the predator and the prey produced through one interaction
are not necessarily independent random variables. However, each of the
$\lambda$ interactions in one generation are independent and
identically distributed random variables.

We will restrict ourselves to solution concepts that can be
characterised as finding a given target subset
$\mathcal{S}\subseteq \X\times \Y$. This captures for example maximin
optimisation or finding pure Nash equilibria. Within this context, the
goal of Algorithm~\ref{algo:coea-process} is now to obtain populations
$P_t$ and $Q_t$ such that their product intersects with the target set
$\mathcal{S}$. We then define the runtime of an algorithm $A$ as the
number of interactions before the target subset
has been found.
\begin{definition}[Runtime]
  For any instance $\mathcal{A}$ of Algorithm~\ref{algo:coea-process} and 
  subset $\mathcal{S}\subseteq \X\times \Y$, define
  $T_{\mathcal{A},\mathcal{S}} := \min \{t\lambda \in\mathbb{N}\mid (P_t\times Q_t)\cap\mathcal{S}\neq\emptyset\}.
  $\end{definition}
We follow the convention in analysis of population-based EAs that the
granularity of the runtime is in generations, i.e., multiples of
$\lambda$. The definition overestimates the number of interactions
before a solution is found by at most $\lambda-1$.

\subsection{Tracking the algorithm state}
\label{section:distr-to-product}

We will now discuss how the state of Algorithm~\ref{algo:coea-process}
can be captured with a stochastic process. To determine the trajectory
of a co-evolutionary algorithm, it is insufficient, naturally, to
track only one of the populations, as the dynamics of the algorithm
is determined by the relationship between the two populations.

Given the definition of runtime, it will be natural to describe the
state of the algorithm via the Cartesian product $P_t\times Q_t$. In
particular, for subsets $A\subset\X$ and $B\subset Y$, we will study
the drift of the stochastic process $Z_t:= \vert (P_t\times Q_t)\cap(A\times B) \vert $.

Naturally, the predator $x$ and the prey $y$ sampled in line
\ref{alg:coev:evaluate2} of Algorithm~\ref{algo:coea-process} are not
necessarily independent random variables. However, a predator $x$
sampled in interaction $i_1$ is probabilistically independent of any prey
sampled in an interaction $i_2\neq i_1$. In order to not
have to explicitly take these dependencies into account later in the
paper, we now characterise properties of the distribution of $Z_t$ in
Lemma~\ref{lemma:marg-prob-to-product-prob}.

\begin{lemma}\label{lemma:marg-prob-to-product-prob}
  Given subsets $A\subset\mathcal{X}, B\subset\mathcal{Y}$, 
  assume that for any $\delta>0$ and $\gamma\in(0,1)$, the sample
  $(x,y)\sim \mathcal{D}(P_t,Q_t)$ satisfies
  $$\prob{x\in A}\prob{y\in B}\geq (1+\delta)\gamma,$$ and $P_t$ and
  $Q_t$ are adapted to a filtration $(\filt{t})_{t\in\mathbb{N}}$.
  Then the random variable $Z_{t+1}:=\vert(P_{t+1}\times Q_{t+1}) \cap
  (A\times B)\vert$ satisfies 
  \begin{description}
   \item[1)] $\expect{Z_{t+1}\mid\filt{t}} \geq
     \lambda(\lambda-1)(1+\delta)\gamma$.
   \item[2)] $\expect{e^{-\eta Z_{t+1}}\mid \filt{t}}\leq
     e^{-\eta\lambda(\gamma\lambda-1)}$ for  $0<\eta\leq
     (1-(1+\delta)^{-1/2})/\lambda$
   \item[3)] $\prob{Z_{t+1}<  \lambda(\gamma\lambda-1)\mid\filt{t} } \leq e^{-\delta_1\gamma\lambda\left(1-\sqrt{\frac{1+\delta_1}{1+\delta}}\right)}$ for $\delta_1\in(0,\delta)$.
   \end{description}
\end{lemma}
\begin{proof}
  In generation $t+1$, the algorithm samples independently and
  identically $\lambda$ pairs
  $(P_{t+1}(i),Q_{t+1}(i))_{i\in[\lambda]}$ from distribution $\mathcal{D}(P_t,Q_t)$.
  For all $i\in[\lambda]$, define the random variables
  $X'_{i} := \indf{P_{t+1}(i)\in A}$ and 
  $Y'_{i} := \indf{Q_{t+1}(i)\in B}$.
  Then since the algorithm samples each pair $(P_{t+1}(i),Q_{t+1}(i))$
  independently, and by the assumption of the lemma, there exists
  $p,q\in(0,1]$ such that
  $X':=\sum_{i=1}^\lambda X'_i\sim \bin(\lambda,p)$, and
  $Y':=\sum_{i=1}^\lambda Y'_i\sim \bin(\lambda,q)$, where $pq\geq \gamma(1+\delta)$.
  By these definitions, it follows that $Z_{t+1}=X'Y'$.
  
  Note that $X'$ and $Y'$ are not necessarily independent random
  variables because $X'_i$ and $Y'_i$ are not necessarily
  independent. However, by defining two independent 
  binomial random variables $X\sim\bin(\lambda,p)$, and
  $Y\sim\bin(\lambda,q)$, we readily have the stochastic dominance
  relation
  \begin{align}
    Z_{t+1} = X'Y' & = \left(\sum_{i=1}^\lambda X'_i\right)\left(\sum_{j=1}^\lambda Y'_j\right)\\
                   & = \left(\sum_{i=1}^\lambda X_i\right)\left(\sum_{j\neq i} Y_j\right) + \sum_{i=1}^\lambda X'_iY'_i\\
                   & = XY - \sum_{i=1}^\lambda X_iY_i + \sum_{i=1}^\lambda X'_iY'_i \\
                   & \succeq XY - \sum_{i=1}^\lambda X_iY_i. \label{eq:z-stoch-dom}
  \end{align}
  
  The first statement of the lemma is now obtained by exploiting
  (\ref{eq:z-stoch-dom}),
  Lemma~\ref{lemma:stoch-dom-implies-expectation} in the appendix,
  and the independence between $X$ and $Y$
  \begin{align*}
    \expectt{Z_{t+1}}
    & \geq \expect{XY-\sum_{i=1}^\lambda X_iY_i}
     = \expect{X}\expect{Y}-\sum_{i=1}^\lambda \expect{X_i}\expect{Y_i}\\
    & = p\lambda q\lambda - \lambda p q
      = pq\lambda(\lambda-1)
     \geq (1+\delta)\gamma\lambda(\lambda-1).
  \end{align*}

  For the second statement, we apply Lemma~\ref{lemma:product-mgf}
  wrt $X$, $Y$, and the parameters $\sigma:=\sqrt{1+\delta}-1$ and $z:=\gamma$.
  By the assumption on $p$ and $q$, we have
  $pq \geq (1+\delta)\gamma = (1+\sigma)^2z,
  $ 
  furthermore the constraint on parameter $\eta$ gives
  \begin{align*}
    \eta & \leq \frac{1}{\lambda}\left(1-\frac{1}{\sqrt{1+\delta}}\right)
           = \frac{\sqrt{1+\delta}-1}{\lambda\sqrt{1+\delta}}
           = \frac{\sigma}{(1+\sigma)\lambda}.
  \end{align*}
  The assumptions of Lemma~\ref{lemma:product-mgf} are satisfied, and
  we obtain from (\ref{eq:z-stoch-dom})
  \begin{align*}
    \expectt{e^{-\eta Z_{t+1}}}
    & \leq \expect{\exp\left(-\eta XY+\eta\sum_{i=1}^\lambda X_iY_i\right)}\\
    & < e^{\eta\lambda}\cdot \expect{e^{-\eta XY}}
     < e^{\eta\lambda}\cdot e^{-\eta\gamma\lambda^2}
     = e^{-\eta\lambda(\gamma\lambda-1)}.
  \end{align*}

  Given the second statement, the third statement will be proved by a
  standard Chernoff-type argument. Define $\delta_2>0$ such that
  $(1+\delta_1)(1+\delta_2)=1+\delta$. For
  \begin{align*}
    \eta := \frac{1}{\lambda}\left(1-\frac{1}{\sqrt{1+\delta_2}}\right)
          = \frac{1}{\lambda}\left(1-\sqrt{\frac{1+\delta_1}{1+\delta}}\right)
  \end{align*}
  and $a:=\lambda(\gamma\lambda-1)$, it
  follows by Markov's inequality
  \begin{align*}
    \probt{Z_{t+1}\leq a}
    & = \probt{e^{-\eta Z_{t+1}}\geq e^{-\eta a}}
      \leq e^{\eta a}\cdot \expectt{e^{-\eta Z_{t+1}}}\\
    & \leq e^{\eta a}\cdot \exp\left(-\eta\lambda(\gamma(1+\delta_1)\lambda-1)\right)\\
    & =    e^{\eta a - \eta a-\eta \gamma\lambda^2\delta_1}
      =    e^{-\eta \gamma\lambda^2\delta_1}\\
    & =    \exp\left(-\delta_1\left(1-\sqrt{\frac{1+\delta_1}{1+\delta}}\right)\gamma\lambda\right),
  \end{align*}
  where the last inequality applies statement 2.
\end{proof}

The next lemma is a variant of
Lemma~\ref{lemma:marg-prob-to-product-prob}, and will be used to
compute the probability of producing individuals in ``new'' parts of
the product space $\X\times\Y$ (see condition (G1) of Theorem
\ref{thm:level-based}).

\begin{lemma}\label{lemma:population-upgrade}
  For $A\subset \mathcal{X}$ and $B\subset\mathcal{Y}$
  define $$r:=\prob{(P_{t+1}\times Q_{t+1})\cap (A\times B)\neq\emptyset}.$$
  If for $(x,y)\sim \D(P_t,Q_t)$, it holds
  $\prob{x\in A}\prob{y\in B} \geq z,
  $ then
  \begin{align*}
    \frac{1}{r} < \frac{3}{z(\lambda-1)}+1.
  \end{align*}
\end{lemma}
\begin{proof}Define $p:= \prob{x\in A}, q:=\prob{y\in B}$ and
  $\lambda':=\lambda-1$. Then by the definition of $r$ and Lemma~\ref{lemma:q-lower-bound}
  \begin{align*}
    r & \geq \prob{\exists k\neq \ell : P_{t+1}(k)=u\wedge Q_{t+1}(\ell)=v }\\
        &   \geq (1-(1-p)^\lambda)(1-(1-q)^{\lambda'})
           >  \left(\frac{\lambda'p}{1+\lambda'p}\right)\left(\frac{\lambda'q}{1+\lambda'q}\right)\\
        &   \geq \frac{\lambda'^2z}{1+\lambda'(p+q)+\lambda'^2z}
           \geq \frac{\lambda'^2z}{1+2\lambda'+\lambda'^2z}.
  \end{align*}
  Finally,
  $$\frac{1}{r} \leq \frac{2}{z\lambda'}+\frac{1}{z\lambda'^2}+1 < \frac{3}{z\lambda'}+1=\frac{3}{z(\lambda-1)}+1.$$
\end{proof}

\section{A Level-based Theorem for Co-Evolutionary Processes} \label{sec:level-based-theorem}

This section provides a generic tool (Theorem~\ref{thm:level-based}),
a level-based theorem for co-evolution, for deriving upper bounds on
the expected runtime of Algorithm~\ref{algo:coea-process}. Since this
theorem can be seen as a generalisation of the original level-based
theorem for classical evolutionary algorithms introduced in
\cite{levelbasedanalysis2018}, we will start by briefly discussing the
original theorem. Informally, it assumes a population-based process
where the next population $P_{t+1}\in\mathcal{X}^\lambda$ is obtained by sampling independently
$\lambda$ times from a distribution $\mathcal{D}(P_t)$ that depends on
the current population $P_t\in\mathcal{X}^\lambda$. The theorem
provides an upper bound on the expected number of generations until
the current population contains an individual in a target set
$A_{\geq m}\subset \mathcal{X}$, given that the following three informally-described conditions
hold. Condition (G1): If a fraction $\gamma_0$ of the population belongs to a
``current level'' (i.e., a subset) $A_{\geq j} \subset \mathcal{X}$, then the distribution
$\mathcal{D}(P_t)$ should assign a non-zero probability $z_j>0$ of
sampling individuals in the ``next level'' $A_{\geq j+1}$. Condition
(G2): If already a $\gamma$-fraction of the population belongs to the
next level $A_{\geq j+1}$ for $\gamma\in(0,\gamma_0)$, then the distribution
$\mathcal{D}(P_t)$ should assign a probability at least
$\gamma(1+\delta)$ to the next level. Condition (G3) is a requirement
on the population size $\lambda$. Together, conditions (G1) and
(G2) ensure that the process ``discovers'' and multiplies on next
levels, thus evolving towards the target set. Due to its generality,
the classical level-based theorem and variations of it have found
numerous applications, e.g., in runtime analysis of genetic algorithms
\cite{corus_theory_2018}, estimation of distribution
algorithms~\cite{lehre_improved_2017}, evolutionary algorithms applied
to uncertain optimisation \cite{dang_runtime_2016}, and evolutionary
algorithms in multi-modal optimisation \cite{dang_escaping_2021,dang_non-elitist_2021}.

We now present the new theorem, a level-based theorem for
co-evolution, which is one of the main contributions of this
paper. The theorem states four conditions (G1), (G2a), (G2b), and (G3)
which when satisfied imply an upper bound on the runtime of the
algorithm. To apply the theorem, it is necessary to provide a sequence
$(A_j\times B_j)_{j\in [m]}$ of subsets of $\X\times\Y$ called levels,
where $A_1\times B_1=\X\times\Y,$ and where $A_m\times B_m$ is the
target set. It is recommended that this sequence overlaps to some
degree with the trajectory of the algorithm. The ``current level'' $j$
corresponds to the latest level occupied by at least a
$\gamma_0$-fraction of the pairs in $P_t\times Q_t$. Condition (G1)
states that the probability of producing a pair in the next level is
strictly positive. Condition (G2a) states that the proportion of pairs
in the next level should increase at least by a multiplicative factor
$1+\delta$. The theorem applies for any positive parameter $\delta$,
and does not assume that $\delta$ is a constant with respect to
$m$. Condition (G2a) implies that the fraction of pairs in the current
level should not decrease below $\gamma_0$. Finally, Condition (G3)
states a requirement in terms of the population size.

In order to make the ``current level'' of the populations well
defined, we need to ensure that for all populations
$P\in\mathcal{X}^\lambda$ and $Q\in\mathcal{Y}^\lambda$, there exists
at least one level $j\in[m]$ such that
$\vert(P\times Q)\cap (A_j\times B_j)\vert\geq \gamma_0\lambda^2$.  This is
ensured by defining the initial level 
$A_1\times B_1:=\mathcal{X}\times\mathcal{Y}$.

Notice that the notion of ``level'' here is more general than in the
classical level-based theorem \cite{levelbasedanalysis2018}, in that
they do not need to form a partition of the search space.

\begin{theorem}\label{thm:level-based}
Given subsets $A_j\subseteq \mathcal{X}$, $B_j\subseteq\mathcal{Y}$ for
$j\in[m]$ where $A_1:=\mathcal{X}$ and $B_1:=\mathcal{Y}$, define
    $T := \min\{t\lambda \mid (P_t\times Q_t)\cap (A_{m}\times B_m)\neq \emptyset\}$, where for all
    $t\in\mathbb{N}$, $P_t\in\mathcal{X}^\lambda$ and
    $Q_t\in\mathcal{Y}^\lambda$ are the 
    populations of Algorithm~\ref{algo:coea-process} in generation $t$.
    If there
    exist $z_1,\dots,z_{m-1},\delta \in(0,1]$,
    and $\gamma_0 \in (0,1)$
    such that
    for any populations $P\in\mathcal{X}^\lambda$ and
    $Q\in\mathcal{Y}^\lambda$ with so-called ``current level'' $j:=\max\{i\in[m]\mid
    \vert(P\times Q)\cap (A_i\times B_i)\vert\geq \gamma_0\lambda^2\}$
    
\begin{description}[noitemsep,leftmargin=3em]
  \item[(G1)]
  if $j\in[m-1]$ and $(x,y)\sim \mathcal{D}(P, Q)$ then
\[
    \displaystyle \prob{x\in A_{j+1}}\prob{y\in B_{j+1}} \geq z_j,
\]
\item[(G2a)]
  for all $\gamma\in(0,\gamma_0)$,
  if $j\in[m-2]$ and
      $\vert(P\times Q) \cap (A_{j+1}\times B_{j+1})\vert  \geq \gamma\lambda^2$, then
for $(x,y)\sim \mathcal{D}(P, Q)$,
      \[
        \prob{x\in A_{j+1}}\prob{y\in B_{j+1}} \geq (1+\delta)\gamma,\]
  \item[(G2b)]
    if $j\in[m-1]$ and $(x,y)\sim \mathcal{D}(P, Q)$, then
      \[
        \prob{x\in A_{j}}\prob{y\in B_{j}} \geq (1+\delta)\gamma_0,\]
  \item[(G3)] and the population size $\lambda\in\mathbb{N}$ satisfies
   for $z_*:=\min_{i\in[m-1]} z_i$ and any constant $\upsilon>0$
\[
\lambda \geq 2\left(\frac{1}{\gamma_0\delta^2}\right)^{1+\upsilon}\ln\left(\frac{m}{z_*}\right),
\]
  \end{description}
  then for any constant $c''>1$, and sufficiently large\footnote{See
    Theorem \ref{thm:level-based-precise-lambda} in the appendix for a precise constant.} $\lambda$,
\begin{align}
\expect{T} \leq \frac{c''\lambda}{\delta}\left(m\lambda^2+ 16\sum_{i=1}^{m-1}\frac{1}{z_i}\right).
\end{align}
\end{theorem}

The proof of Theorem~\ref{thm:level-based} uses drift
analysis, and follows closely
the proof of the original level-based theorem
\cite{levelbasedanalysis2018}, however there are some notable
differences, particularly in the assumptions about the underlying
stochastic process and the choice of the ``level functions''. For ease
of comparison, we have kept the proof identical to the classical proof
where possible. We first recall the notion of a level-function which
is used to glue together two distance functions in the drift analysis.

\begin{definition}[\cite{levelbasedanalysis2018}]\label{def:property}
  For any $\lambda,m\in\mathbb{N}\setminus\{0\}$,
  a function $g:[0..\lambda^2]\times [m]\rightarrow\mathbb{R}$ is called
  a \emph{level function} if the following three conditions hold
  \begin{enumerate}
  \item $\forall x\in[0..\lambda^2], \forall y\in[m-1],
         g(x,y) \geq g(x,y+1)$,
  \item $\forall x\in\cup[0..\lambda^2-1], \forall y\in[m],
         g(x,y)\geq g(x+1,y)$, and
  \item $\forall y\in[m-1],
         g(\lambda^2,y)\geq g(0,y+1)$.
  \end{enumerate}
\end{definition}

It follows directly from the definition that the set of level
functions is closed under addition. More precisely, for any pair of
level functions $g,h:[0..\lambda^2]\times [m]\rightarrow\mathbb{R}$,
the function $f(x,y):=g(x,y)+h(x,y)$ is also a level function.
The proof of Theorem \ref{thm:level-based} defines one process 
$(Y_t)_{t\in\mathbb{N}}\in[m]$ which informally corresponds to the ``current level''
of the process in generation $t$, and a sequence of $m$ processes $(X^{(1)}_t)_{t\in\mathbb{N}},\ldots,(X^{(m)}_t)_{t\in\mathbb{N}}$,
$j\in[m]$, where informally $X^{(j)}_t$ refers to the number of individuals above
level $j$ in generation $t$. Thus, $X^{(Y_{t})}_t$ corresponds to the
number of individuals above the current level in generation $t$. 
A level-function $g$ and the following lemma will be used to define a
global distance function used in the drift analysis.

\begin{lemma}[\cite{levelbasedanalysis2018}]\label{lemma:increase}
  If $Y_{t+1}\geq Y_t,$ then for any level function $g$
  \begin{align*}
         g\left(X^{(Y_{t+1}+1)}_{t+1},Y_{t+1}\right)
    \leq g\left(X^{(Y_t+1)}_{t+1},Y_t\right).
  \end{align*}
\end{lemma}
\begin{proof}
  The statement is trivially true when $Y_t=Y_{t+1}$. On the other hand,
  if $Y_{t+1}\geq Y_t+1$, then the conditions in Definition~\ref{def:property} imply
  \begin{align*}
      g\left(X^{(Y_{t+1}+1)}_{t+1},Y_{t+1}\right)
      & \leq g\left(0,Y_{t+1}\right)
        \leq g\left(0,Y_{t}+1\right)\\
      & \leq g\left(\lambda^2,Y_{t}\right)
        \leq g\left(X^{(Y_t+1)}_{t+1},Y_t\right). 
  \end{align*}
\end{proof}

We now proceed with the proof of the level-based theorem for
co-evolutionary processes.

\begin{proof}[Proof of Theorem \ref{thm:level-based}]
We apply
Theorem~\ref{thm:pol-drift} (the additive drift theorem)
with respect to the parameter $a=0$ and the
process
$
  Z_t := g\left(X_{t}^{(Y_t+1)},Y_t\right),
$
where $g$ is a level-function, and
$(Y_t)_{t\in\mathbb{N}}$ and $(X^{(j)}_t)_{t\in\mathbb{N}}$ for
$j\in[m]$ are stochastic processes, which will be defined later.
$(\mathscr{F}_t)_{t\in\mathbb{N}}$ is the filtration induced by
the populations $(P_t)_{t\in\mathbb{N}}$ and $(Q_t)_{t\in\mathbb{N}}$.

We will assume \WLOG that condition (G2a) is also satisfied for
$j=m-1$, for the following reason.
Given Algorithm~\ref{algo:coea-process} with a certain mapping~$\D$, consider
Algorithm~1 with a modified mapping~$\D'(P,Q)$:
If $(P\times Q)\cap (A_{m}\times B_m)=\emptyset$, then $\D'(P,Q)=\D(P,Q)$; otherwise $\D'(P,Q)$
assigns probability mass $1$ to some pair $(x,y)$ of~$P\times Q$ that is in
$A_{m}$, \eg, to the first one among such elements.
Note that $\D'$ meets conditions~(G1), (G2a), and (G2b). Moreover, (G2a) 
hold for $j=m-1$.
For the sequence of populations $P'_0,P'_1,\dots$ and $Q'_0,Q'_1,\dots$ of Algorithm~\ref{algo:coea-process}
with mapping~$\D'$, we can put ${T' := \min\{\lambda t \mid
(P'_t\times Q'_t)\cap (A_{m}\times B_m) \neq \emptyset\}}$. Executions of the original algorithm and
the modified one before generation $T'/\lambda$ are identical. On
generation~$T'/\lambda$ both algorithms place elements of~$A_{m}$
into the populations for the first time. Thus, $T'$  and $T$ are
equal in every realisation and their expectations are equal.

  For any level $j\in[m]$ and time $t\geq 0$, let the random variable
$
    X_t^{(j)} := \vert (P_t\times Q_t) \cap (A_j\times B_j) \vert
$
  denote the number of pairs in level $A_{j}\times B_j$ at time $t$.
  As mentioned above, the current level $Y_t$ of the algorithm at time $t$
  is defined as
  \begin{align*}
    Y_t & := \max \left\{ j\in[m] \;\mid\; X_t^{(j)} \geq  \gamma_0\lambda^2 \right\}.
    \end{align*}
  Note that  $(X^{(j)}_t)_{t\in\mathbb{N}}$ and
  $(Y_t)_{t\in\mathbb{N}}$ are adapted to the filtration
  $(\mathscr{F}_t)_{t\in\mathbb{N}}$ because they are defined in terms
  of the populations $(P_t)_{t\in\mathbb{N}}$ and $(Q_t)_{t\in\mathbb{N}}$.

  When $Y_t <m$, there exists a unique $\gamma\in[0,\gamma_0)$ such that
  \begin{align}
    X_t^{(Y_t+1)} & = \vert(P_t\times Q_t)\cap (A_{Y_t+1}\times
                    B_{Y_t+1})\vert =  \gamma \lambda^2, \text{ and}\label{eq:config-3}\\
    X_t^{(Y_t)}   & = \vert(P_t\times Q_t)\cap (A_{Y_t}\times B_{Y_t})\vert \geq \gamma_0\lambda^2. \label{eq:config-2}
  \end{align}

  Finally, we define the process $(Z_t)_{t\in\mathbb{N}}$ as
  $Z_t:=0$ if $Y_t=m$, and otherwise, if $Y_t<m$, we let
  $$
    Z_t:=g\left(X_{t}^{(Y_t+1)},Y_t\right),
  $$
  where for all $k\in[\lambda^2]$, and for all $j\in[m-1]$,
  $g(k,j):=g_1(k,j)+g_2(k,j)$ and
  \begin{align*}
    g_1(k,j) &:= \frac{\eta}{1+\eta}\cdot ((m-j)\lambda^2-k)\\
    g_2(k,j) &:= \varphi\cdot\left(\frac{e^{-\eta k}}{q_{j} } + \sum^{m-1}_{i=j+1} \frac{1}{q_i}\right),
  \end{align*}
  where $\eta\in(3\delta/(11\lambda),\delta/(2\lambda))$
  and $\varphi\in(0,1)$ are parameters which will be specified later, and 
  for $j\in[m-1]$, 
  $q_j :=  \lambda z_j/(4+\lambda z_j).
  $

  Both functions have partial derivatives $\frac{\partial g_i}{\partial k}<0$ and
  $\frac{\partial g_i}{\partial j}<0$, hence they satisfy properties 1
  and 2 of Definition~\ref{def:property}. They also satisfy property 3 because
  for all $j\in[m-1]$
  \begin{align*}
    g_1(\lambda^2,j) & = \frac{\eta}{1+\eta}((m-j)\lambda^2-\lambda^2) = g_1(0,j+1)\\
    g_2(\lambda^2,j) & > \sum^{m-1}_{i=j+1} \frac{\varphi}{q_i} = g_2(0,j+1).
  \end{align*}
  Therefore $g_1$ and $g_2$ are level functions, and thus also their
  linear combination $g$ is a level function.
  
  Due to properties 1 and 2 of level functions (see Definition~\ref{def:property}),
  it holds for all $k\in [0..\lambda^2]$ and $j\in [m-1]$
  \begin{align}
    0\leq g(k,j)\leq g(0,1) & = \frac{\eta(m-1)\lambda^2}{1+\eta} +\varphi\cdot\left(\frac{1}{q_{1} } + \sum^{m-1}_{i=2} \frac{1}{q_i}\right)\\
                             & <  \frac{\eta m\lambda^2}{1+\eta} +\sum_{i=1}^{m-1} \frac{\varphi}{q_i}\\
                            & < \frac{\eta m\lambda^2}{1+\eta}+
                               \varphi\sum_{i=1}^{m-1}\frac{4+\lambda z_i}{\lambda z_i}    \label{eq:distance-bound-0}\\
\intertext{using $\eta>0$}
                            & < m\left(\eta\lambda^2+ \frac{4\varphi}{\lambda z_*}+\varphi\right)\\
                             \intertext{using $\varphi,z_*\in(0,1)$ and  $\lambda>11\eta/(3\delta)$}
                             & < \frac{m}{z_*}\left(2\eta\lambda^2+ \frac{44\eta}{3\delta}\right)
                               \intertext{assuming 
                               $\lambda>44/3$ and  using $\lambda^2>\lambda
                               \delta^{-2(1+\upsilon)}>44/(3\delta)$                               
                               }
                             & < \frac{3\eta\lambda^2m}{z_*}.
    \label{eq:distance-bound-1}
  \end{align}
  Hence, we have $0\leq
  Z_t<g(0,1)<\infty$ for all $t\in\mathbb{N}$ which implies that
  condition 2 of the drift theorem is satisfied.

  The drift of the process at time $t$ is $\expectt{\Delta_{t+1}}$, where
  \begin{align*}
    \Delta_{t+1}  & := g\left(X_t^{(Y_t+1)},Y_t\right)-g\left(X_{t+1}^{(Y_{t+1}+1)},Y_{t+1}\right).
  \end{align*}
We bound the drift by the law of total probability as
  \begin{align}
  \expectt{\Delta_{t+1}}
         &  =  (1-\probt{Y_{t+1}<Y_t})\expectt{\Delta_{t+1} \mid Y_{t+1}\geq Y_t} \nonumber \\
         &\quad\; + \probt{Y_{t+1}<Y_t}\expectt{\Delta_{t+1} \mid Y_{t+1}< Y_t}.\label{eq:law-tot-prob}
  \end{align}
  The event $Y_{t+1}< Y_t$ holds if and only if
  $X_{t+1}^{(Y_t)}<\gamma_0\lambda^2$, which by
      Lemma~\ref{lemma:marg-prob-to-product-prob} statement~3 for
      $\gamma:=\gamma_0+1/\lambda$ and a parameter
      $\delta_1\in(0,\delta)$ to be chosen later, and
      conditions~(G2b) and (G3), is upper bounded by
  \begin{align}
    \probt{Y_{t+1}<Y_t}
    &= \probt{X_{t+1}^{(Y_t)}<\gamma_0\lambda^2}  \\
    &= \probt{X_{t+1}^{(Y_t)}<\lambda(\gamma\lambda-1)}  \\
    & < \exp\left(-\delta_1\gamma\lambda\left(1-\sqrt{\frac{1+\delta_1}{1+\delta}}\right)\right)\\
    \intertext{by Lemma~\ref{lemma:sqrt-bound} and $\gamma<\gamma_0$}
    & < \exp\left(-\delta_1\gamma_0\lambda\left(\frac{3\delta-4\delta_1}{11}\right)\right)\\
    \intertext{to minimise the expression, we choose $\delta_1:=(3/8)\delta$}
    & = \exp\left(-\frac{9}{16}\delta^2\gamma_0\lambda\right).
\label{eq:prob-fall}
  \end{align}
Given the low probability of the event $Y_{t+1}<Y_t$, it suffices
  to use the pessimistic bound~(\ref{eq:distance-bound-1}) 
  \begin{align}
    \expectt{\Delta_{t+1} \mid Y_{t+1}<Y_t}
     & \geq -g(0,1) 
       \label{eq:drift-fall}\end{align}

  If $Y_{t+1}\geq Y_t$, we can apply Lemma~\ref{lemma:increase}
  \begin{align*}
    &\expectt{\Delta_{t+1} \mid Y_{t+1}\geq Y_t} \geq
      \expectt{g\left(X^{(Y_t+1)}_{t},Y_t\right) - g\left(X^{(Y_t+1)}_{t+1},Y_{t}\right)
               \mid Y_{t+1}\geq Y_t}.
  \end{align*}

  If $X_{t}^{(Y_t+1)}=0$, then $X_{t}^{(Y_t+1)}\leq X_{t+1}^{(Y_t+1)}$ and
  \begin{align*}
    \expectt{g_1\left(X_t^{(Y_t+1)},Y_t\right) - g_1\left(X_{t+1}^{(Y_t+1)},Y_t\right)
             \mid Y_{t+1}\geq Y_t}\geq 0,
  \end{align*}
  because the function $g_1$ satisfies property~2 in Definition~\ref{def:property}.
  Furthermore, we have the lower bound
  \begin{multline*}
    \expectt{g_2\left(X_t^{(Y_t+1)},Y_t\right)-g_2\left(X_{t+1}^{(Y_t+1)},Y_t\right)
             \mid Y_{t+1}\geq Y_t} \\
      >    \probt{X_{t+1}^{(Y_t+1)}\geq 1}
           \left(g_2\left(0,Y_t\right)-g_2\left(1,Y_t\right)\right)
      \geq \frac{\eta\varphi}{1+\eta}.
  \end{multline*}
  where the last inequality follows because 
  \begin{align*}
  \probt{X_{t+1}^{(Y_t+1)}\geq 1}
    &= \probt{(P_{t+1}\times Q_{t+1})\cap (A_{Y_t+1}\times B_{Y_t+1})\neq \emptyset}\\
    & \geq q_{Y_t},
  \end{align*}
  due to condition (G1) and Lemma~\ref{lemma:population-upgrade},
  and
  $$ g_2\left(0,Y_t\right)-g_2\left(1,Y_t\right)
  = (\varphi/q_{Y_t})(1-e^{-\eta})
  \geq \frac{\varphi\eta}{(1+\eta)q_{Y_t}}
  $$

      In the other case, where $X_t^{(Y_t+1)}=\gamma\lambda^2\geq 1$,      
  Lemma~\ref{lemma:marg-prob-to-product-prob} and condition (G2a)
  imply for $\varphi:=\delta(1-\delta')$ for an arbitrary constant $\delta'\in(0,1)$,
  \begin{multline}
    \expectt{g_1\left(X_t^{(Y_t+1)},Y_t\right)-g_1\left(X_{t+1}^{(Y_t+1)},Y_t\right)
              \mid Y_{t+1} \geq Y_{t}} \\
      = \frac{\eta}{1+\eta}\expectt{X_{t+1}^{(Y_t+1)}\mid Y_{t+1} \geq Y_{t}}-\frac{\eta}{1+\eta} X_t^{(Y_t+1)}\\
      \geq \frac{\eta}{1+\eta}(\lambda(\lambda-1)(1+\delta)\gamma-\gamma\lambda^2) > \frac{\eta}{1+\eta}\delta(1-\delta')=\frac{\eta\varphi}{1+\eta},\label{eq:def-phi}
  \end{multline}
  where the last inequality is obtained by choosing the minimal value $\gamma=1/\lambda^2$.
  For the function $g_2$, we get
  \begin{multline*}
    \expectt{g_2\left(X_t^{(Y_t+1)},Y_t\right)-g_2\left(X_{t+1}^{(Y_t+1)},Y_t\right)
             \mid Y_{t+1} \geq Y_{t}}
      = \\
    \frac{\varphi}{q_{Y_t}}
    \left( e^{-\eta X_t^{(Y_t+1)}} - \expectt{e^{-\eta X_{t+1}^{(Y_t+1)}}} \right)>0,
  \end{multline*}
  where the last inequality is due to statement 2 of Lemma~\ref{lemma:marg-prob-to-product-prob}
      for the parameter
      \begin{align*}
        \eta := \frac{1}{\lambda}\left(1-\frac{1}{\sqrt{1+\delta}}\right).
      \end{align*}
      By Lemma \ref{lemma:sqrt-bound} for $\delta_1=0$, this parameter satisfies
      \begin{align}
         \frac{3\delta}{11\lambda}<\eta < \frac{\delta}{2\lambda} < \frac{1}{\lambda}.
        \label{eq:eta-bound}
      \end{align}
  
  Taking into account all cases, we have
  \begin{align}
    \expectt{\Delta_{t+1} \mid Y_{t+1}\geq Y_t }
      \geq \frac{\eta\varphi}{1+\eta}.\label{eq:drift-forward}
  \end{align}

  We now have bounds for all the quantities in~\eqref{eq:law-tot-prob}
  with \eqref{eq:prob-fall}, \eqref{eq:drift-fall}, and \eqref{eq:drift-forward}.
  Before bounding the overall drift $\expectt{\Delta_{t+1}}$, we
  remark that the requirement on the
  population size imposed by condition (G3) implies that for any constants
$\upsilon>0$ and $C>0$, and sufficiently large $\lambda$,
  \begin{align*}
    \frac{\lambda}{16C\ln\lambda} > \lambda^{\frac{1}{1+\upsilon}}> \frac{1}{\delta^2\gamma_0},
  \end{align*}
  which implies that 
\begin{align}
  C\ln\lambda < \frac{\lambda\delta^2\gamma_0}{16}. \label{eq:loglambda-bound}
\end{align}
  The overall drift is now bounded by
  \begin{align}
  \expectt{\Delta_{t+1}}
    &=       (1 - \probt{Y_{t+1}<Y_t})\expectt{\Delta_{t+1} \mid Y_{t+1}\geq Y_t} \\
    &\quad\;    + \probt{Y_{t+1}<Y_t}\expectt{\Delta_{t+1} \mid Y_{t+1}< Y_t} \\
    & \geq \frac{\eta\varphi}{1+\eta} - \exp\left(-\frac{9}{16}\delta^2\gamma_0\lambda\right)\left(\frac{3m\eta\lambda^2}{z_*}+\frac{\eta\varphi}{1+\eta}\right)\\
    & =
      \frac{\eta\varphi}{1+\eta} -
      \exp\left(-\frac{9}{16}\delta^2\gamma_0\lambda+C\ln\lambda\right)
      \left(\frac{3m\eta\lambda^{2-C}}{z_*} +\frac{\eta\varphi}{(1+\eta)\lambda^C}\right)\\
    \intertext{by (\ref{eq:loglambda-bound})}
    & >
      \frac{\eta\varphi}{1+\eta} -
      \exp\left(-\frac{1}{2}\delta^2\gamma_0\lambda\right)
      \left(\frac{3m\eta\lambda^{2-C}}{z_*} +\frac{\eta\varphi}{(1+\eta)\lambda^C}\right)\\
    \intertext{by condition (G3)}
    & >
      \frac{\eta\varphi}{1+\eta} -(\frac{z_*}{m})      
      \left(\frac{3m\eta\lambda^{2-C}}{z_*} +\frac{\eta\varphi}{(1+\eta)\lambda^C}\right)\\
    \intertext{choosing $C=3$}
    & =
      \frac{\eta\varphi}{1+\eta} - \frac{3\eta}{\lambda} -\frac{\eta\varphi}{(1+\eta)\lambda^3m}\\
    \intertext{by }
    & >
      \frac{\eta\varphi}{1+\eta} - \frac{3\eta\delta}{\sqrt{\lambda}} -\frac{\eta\varphi}{(1+\eta)\lambda^3m}\\
    \intertext{finally, by noting that $1+\eta<1+1/\lambda$ from
    Eq. (\ref{eq:eta-bound}) and that $\varphi=\delta(1-\delta')$ for
    a constant $\delta'\in(0,1)$ mean that for any constant
    $\rho\in(0,1)$, for sufficiently large $\lambda$}
    & > \frac{\eta\varphi(1-\rho) }{1+\eta}.\label{eq:overall-drift}
  \end{align}

  We now verify condition 3 of Theorem~\ref{thm:pol-drift}, \ie,
  that $T$ has
  finite expectation. Let $p_*:=\min\{(1+\delta)(1/\lambda^2), z_*\}>0$, and
  note by conditions (G1) and (G2a) that the current level increases by at
  least one with probability
  $\probt{Y_{t+1}>Y_t}\geq (p_*)^{\gamma_0\lambda}$.
  Due to the definition of the modified process $D'$, if $Y_t=m$, then $Y_{t+1}=m$.
  Hence, the probability of reaching $Y_t=m$ is lower bounded by the
  probability of the event that the current level increases in all of
  at most $m$ consecutive
  generations, \ie, $\probt{Y_{t+m} =m} \geq (p_*)^{\gamma_0\lambda m}>0$.
  It follows that $\expect{T}<\infty$.

  By Theorem~\ref{thm:pol-drift}, the upper bound on $g(0,1)$ in \eqref{eq:distance-bound-0} 
  and the lower bound on the drift in Eq. (\ref{eq:overall-drift}) and the definition of $T$,
  \begin{align*}
    \expect{T}
    &\leq \frac{\lambda(1+\eta)g(0,1)}{\eta\varphi(1-\rho)}\\
    & < \frac{\lambda(1+\eta)}{\eta\varphi(1-\rho)}\left(\frac{\eta m\lambda^2}{1+\eta}+\varphi\sum_{i=1}^{m-1}\frac{4+\lambda z_i}{\lambda z_i}\right)\\
    & < \frac{\lambda}{(1-\rho)}\left(\frac{m\lambda^2}{\varphi}+\frac{1+\eta}{\eta}\sum_{i=1}^{m-1}\left(\frac{4}{\lambda z_i}+1\right)\right)\\
    \intertext{using Eq. (\ref{eq:eta-bound}) and $\varphi:=\delta(1-\delta')$}
    & < \frac{\lambda}{(1-\rho)}\left(\frac{m\lambda^2}{\delta(1-\delta')}+\left(\frac{11\lambda}{3\delta}+1\right)\sum_{i=1}^{m-1}\left(\frac{4}{\lambda z_i}+1\right)\right)\\    
    \intertext{noting that $1<1/\delta\leq\lambda/(3\delta)$ for  $\lambda\geq 3$}
    & < \frac{\lambda}{(1-\rho)}\left(\frac{m\lambda^2}{\delta(1-\delta')}+\left(\frac{4\lambda}{\delta}\right)\sum_{i=1}^{m-1}\left(\frac{4}{\lambda z_i}+1\right)\right)\\    
    & = \frac{\lambda}{(1-\rho)}\left(\frac{m\lambda^2}{\delta(1-\delta')}+\frac{4\lambda(m-1)}{\delta}+\left(\frac{16}{\delta}\right)\sum_{i=1}^{m-1}\frac{1}{z_i}\right)\\        
    \intertext{since $\delta'$ is a constant with respect to
    $\lambda$, for large $\lambda$, this is upper bounded
    by}
    & < \frac{\lambda}{(1-\rho)\delta}\left(\frac{m\lambda^2}{(1-\delta')^2}+16\sum_{i=1}^{m-1}\frac{1}{z_i}\right)\\    
    \intertext{for any constant $c''>1$, we can choose the
    constants $\rho$
    and $\delta'$ such that $c''>(1-\rho)^{-1}(1-\delta')^{-2}$}
    & < \frac{c''\lambda}{\delta}\left(m\lambda^2+16\sum_{i=1}^{m-1}\frac{1}{z_i}\right).
  \end{align*}
\end{proof}

\section{Maximin Optimisation of Bilinear Functions}\label{sec:probl-maxim-optim}
\subsection{Maximin Optimisation Problems}

This section introduces maximin-optimisation problems which is an
important domain for competitive co-evolutionary algorithms \cite{jensen_new_2004,al-dujaili_application_2019,miyagi_adaptive_2021}. We will
then describe a class of maximin-optimisation problems called \bilinear.

It is a common scenario in real-world optimisation that the quality of
candidate solutions depend on the actions taken by some adversary.
Formally, we can assume that there exists a function
\begin{align*}
  g:\mathcal{X}\times\mathcal{Y}\rightarrow\mathbb{R},
\end{align*}
where $g(x,y)$ represents the ``quality'' of solution $x$
when the adversary takes action $y$.

A cautious approach to such a scenario is to search for the candidate
solution which maximises the objective, assuming that the adversary
takes the least favourable action for that solution. Formally, this
corresponds to the \emph{maximin optimisation problem}, i.e., to
maximise the function
\begin{align}
  f(x) := \min_{y\in\mathcal{Y}}  g(x,y).\label{eq:maximin-problem}  
\end{align}
It is desirable to design good algorithms for such problems because
they have important applications in economics, computer science,
machine learning (GANs), and other disciplines.

However, maximin-optimisation problems are computationally challenging
because to accurately evaluate the function $f(x)$, it is necessary to
solve a minimisation problem. Rather than evaluating $f$ directly, the
common approach is to simultaneously maximise $g(x,y)$ with respect to
$x$, while minimising $g(x,y)$ with respect to $y$. For example, if
the gradient of $g$ is available, it is popular to do gradient
ascent-gradient descent.

Following conventions in theory of evolutionary computation
\cite{droste_upper_2006}, we will assume that an algorithm has
\emph{oracle access} to the function $g$. This means that the
algorithm can evaluate the function $g(x,y)$ for any selected pair of
arguments $(x,y)\in\X\times\Y$, however it does not have access to any
other information about $g$, including its definition or the
derivative.  Furthermore, we will assume that $\X=\Y=\{0,1\}^n$, i.e.,
the set of bitstrings of length $n$. While other spaces could be
considered, this choice aligns well with existing runtime analyses of
evolutionary algorithms in discrete domains \cite{wegener_methods_2000,RuntimeAnalysis2020Book}.
To develop a co-evolutionary algorithm for maximin-optimisation, we will
rely on the following dominance relation on the set of pairs
$\mathcal{X}\times\mathcal{Y}$.

\begin{definition}\label{def:domination}
  Given a function
  $g:\mathcal{X}\times\mathcal{Y}\rightarrow\mathbb{R}$ and two pairs
  $(x_1,y_1),(x_2,y_2)\in\mathcal{X}\times\mathcal{Y}$, we say that
  $(x_1,y_1)$ dominates $(x_2,y_2)$ wrt $g$, denoted $(x_1,y_1)\succeq_g
  (x_2,y_2),$ if and only if
  $$g(x_1,y_2) \geq g(x_1,y_1) \geq g(x_2,y_1).$$
\end{definition}

\subsection{The \bilinear Problem}

In order to develop appropriate analytical tools to analyse the
runtime of evolutionary algorithms, it is necessary to start the
analysis with simple and well-understood problems \cite{wegener_methods_2000}.
We therefore define a simple class of a maximin-optimisation
problems that has a particular clear structure. The maximin function 
is defined for two parameters $\alpha,\beta\in[0,1]$ by
\begin{align}
  \text{\bilinear}(x,y) & :=  \|y\|(\|x\|-\beta n)-\alpha n \|x\| \label{eq:simple-bilinear},
\end{align}
where we recall that for any bitstring $z\in\{0,1\}^n$, $\|z\|:=\sum_{i=1}^n z_i$
denotes the number of 1-bits in $z$. The function is illustrated in
Figure~\ref{fig:simple-bilinear1} (left). Extended to the real domain, it
is clear that the function is concave-convex, because $f(x)=g(x,y)$ is
concave (linear) for all $y$, and $h(y)=g(x,y)$ is convex (linear) for
all $x$. The gradient of the function is
$\nabla g = (\|y\| - \alpha n,\|x\| - \beta n)$.
Clearly, we have $\nabla g=0$ when $\|x\|=\beta n$ and $\|y\|=\alpha n$.

\begin{figure*}
  \centering
  \includegraphics[width=0.45\textwidth]{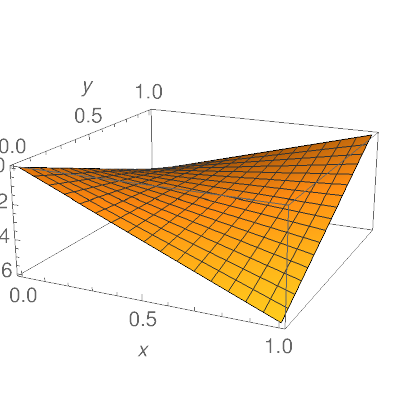}
\includegraphics[width=0.5\textwidth]{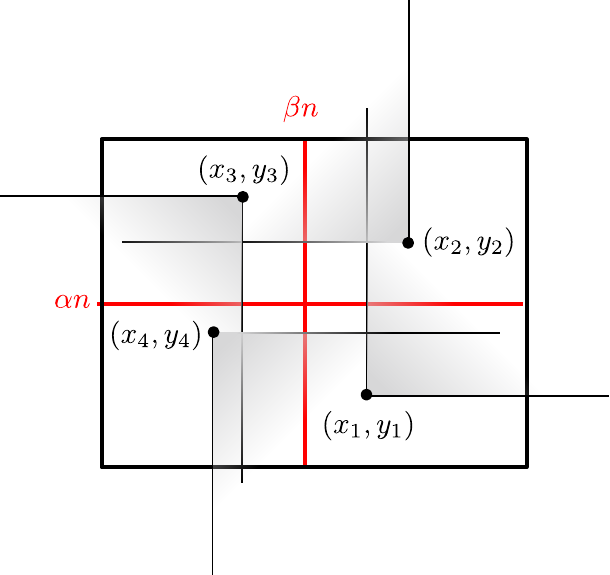}
  \caption{Left: \bilinear for $\alpha=0.4$ and
    $\beta=0.6$.\label{fig:simple-bilinear1} Right: Dominance relationships in \bilinear.}
\end{figure*}

Assuming that the prey (in \Y) always responds with an optimal decision
for every $x\in X$, the predator is faced with the unimodal function $f$ 
below which has maximum when $\|x\|=\beta n$.
\begin{align*}
  f(x) := \min_{y\in\{0,1\}^n} g(x,y) =
  \begin{cases}
    \|x\|(1-\alpha n)-\beta n&
    \text{if } \|x\|\leq\beta n\\
    -\alpha n\|x\|  & \text{if } \|x\|> \beta n.
  \end{cases}  
\end{align*}
The special case where $\alpha=0$ and $\beta=1$ gives
$f(x) = \onemax(x)-n$, i.e., the function $f$ is essentially
equivalent to \onemax, one of the most studied objective functions in
runtime analysis of evolutionary algorithms
\cite{RuntimeAnalysis2020Book}.

We now characterise the dominated solutions wrt \bilinear.

\begin{lemma}\label{lemma:dom-charact}
  Let $g:= $\bilinear. 
  For all pairs $(x_1,y_1),(x_2,y_2)\in \mathcal{X}\times\mathcal{Y}$,
  $(x_1,y_1)\succeq_g (x_2,y_2)$
  if and only if
  \begin{align*}
    \|y_2\|(\|x_1\|-\beta n)\; \geq\;  \|y_1\|(\|x_1\|-\beta n) &\quad \wedge\\
    \|x_1\|(\|y_1\|-\alpha n)\; \geq\;  \|x_2\|(\|y_1\|-\alpha n)&.
  \end{align*}  
\end{lemma}
\begin{proof}
  The proof follows from the definition of $\succeq_g$ and $g$:
  \begin{align*}
                        & g(x_1,y_2)\geq g(x_1,y_1)\\
    \Longleftrightarrow\quad & \|x_1\|\|y_2\|-\alpha n\|x_1\|-\beta n\|y_2\|\geq \|x_1\|\|y_1\|-\alpha n\|x_1\|-\beta n\|y_1\|\\
    \Longleftrightarrow\quad & \|y_2\|(\|x_1\|-\beta n)\geq \|y_1\|(\|x_1\|-\beta n).
  \end{align*}
  The second part follows analogously from $g(x_1,y_1)\geq g(x_2,y_1)$.
\end{proof}

Figure~\ref{fig:simple-bilinear1} (right) illustrates
Lemma~\ref{lemma:dom-charact}, where the $x$-axis and $y$-axis
correspond to the number of 1-bits in the predator $x$, respectively
the number of 1-bits in the prey $y$. The figure contains four pairs,
where the shaded area corresponds to the parts dominated
by that pair: The pair $(x_1,y_1)$ dominates $(x_2,y_2)$,
the pair $(x_2,y_2)$ dominates $(x_3,y_3)$, the pair $(x_3,y_3)$
dominates $(x_4,y_4)$, and the pair $(x_4,y_4)$ dominates
$(x_1,y_1)$. This illustrates that the dominance-relation is intransitive.
Lemma~\ref{lemma:intransitive} states this and other properties of $\succeq_g$.

\begin{lemma}\label{lemma:intransitive}
  The relation $\succeq_g$ is reflexive, antisymmetric, and
  intransitive for $g=$ \bilinear.
\end{lemma}
\begin{proof}
  Reflexivity follows directly from the definition. Assume that
  $(x_1,y_1)\succeq_g (x_2,y_2)$ and $(x_1,y_1)\neq (x_2,y_2)$. Then,
  either $g(x_1,y_2)>g(x_1,y_2),$ or $g(x_1,y_1)>g(x_2,y_1)$, or
  both. Hence, $(x_2,y_2)\not\succeq_g (x_1,y_1)$, which proves that
  the relation is antisymmetric.
  
  To prove intransitivity, it can be shown for any $\varepsilon>0,$
  that $p_1\succeq_g p_2\succeq_g p_3\succeq_g p_2 \succeq_g p_1$
  where
  \begin{align*}
    p_1 & = (\beta+\varepsilon, \alpha-2\varepsilon) &
    p_2 & = (\beta-2\varepsilon, \alpha-\varepsilon) \\
    p_3 & = (\beta-\varepsilon, \alpha+2\varepsilon) &    
    p_4 & = (\beta+2\varepsilon, \alpha+\varepsilon).
  \end{align*}  
\end{proof}

We will frequently use the following simple lemma, which follows 
from the dominance relation and the definition of \bilinear. 
\begin{lemma}\label{lemma:half-prob}
  For \bilinear, and any pairs of populations
  $P\in\mathcal{X}^\lambda, Q\in\mathcal{Y}^\lambda$, consider two samples
  $(x_1,y_1),(x_2,y_2)\sim\unif(P\times Q)$ . Then the following
  conditional probabilities hold.
  \begin{align*}
    \prob{(x_1,y_1)\succeq (x_2,y_2)\mid y_1\leq y_2\wedge x_1>\beta n\wedge x_2>\beta n} & \geq 1/2\\
    \prob{(x_1,y_1)\succeq (x_2,y_2)\mid y_1\geq y_2\wedge x_1<\beta n\wedge x_2<\beta n} & \geq 1/2\\
    \prob{(x_1,y_1)\succeq (x_2,y_2)\mid x_1\geq x_2\wedge y_1>\alpha n\wedge y_2>\alpha n} & \geq 1/2\\
    \prob{(x_1,y_1)\succeq (x_2,y_2)\mid x_1\leq x_2\wedge y_1<\alpha n\wedge y_2<\alpha n} & \geq 1/2.
  \end{align*}
\end{lemma}
\begin{proof}
  All the statements can be proved analogously, so we only show the
  first statement. If $y_1\leq y_2$ and $x_1>\beta n$, $x_2>\beta n$,
  then by Lemma~\ref{lemma:dom-charact}, $(x_1,y_1)\succeq (x_2,y_2)$ if
  and only if $x_1\leq x_2$.

  Since $x_1$ and $x_2$ are independent samples from the same
  (conditional) distribution, it follows that
  \begin{align}
    1 & \geq \prob{x_1>x_2} + \prob{x_1<x_2} = 2\prob{x_1>x_2}
  \end{align}
  Hence, we get
    $\prob{x_1\leq x_2} = 1-\prob{x_1>x_2} \geq 1-1/2 = 1/2$.
\end{proof}

\section{A co-Evolutionary Algorithm for Maximin Optimisation}\label{sec:co-evol-algor}

We now introduce a co-evolutionary algorithm for maximin
optimisation (see Algorithm~\ref{alg:pdcoea}).

The predator and prey populations of size $\lambda$ each are
initialised uniformly at random in lines \ref{alg:pdcoea:init1}--\ref{alg:pdcoea:init2}.
Lines~\ref{alg:pdcoea:evaluate1}--\ref{alg:pdcoea:evaluate3} describe
how each pair of predator and prey are produced, first by selecting a
predator-prey pair from the population, then applying mutation.  In
particular, the algorithm selects uniformly at random two predators
$x_1,x_2$ and two prey $y_1,y_2$ in lines
\ref{label:alg:pdcoea:selection1}--\ref{label:alg:pdcoea:selection2}.
The first pair $(x_1,y_1)$ is \emph{selected} if it dominates the
second pair $(x_2,y_2)$, otherwise the second pair is selected. The
selected predator and prey are mutated by standard bitwise
mutation in lines
\ref{label:alg:pdcoea:mutation1}--\ref{label:alg:pdcoea:mutation2},
i.e., each bit flips independently with probability $\chi/n$ (see
Section C3.2.1 in \cite{ECHandbook}).
The algorithm is a special case of the co-evolutionary
framework in Section~\ref{sec:co-evolutionary-framework}, where
line \ref{alg:coev:evaluate2} in Algorithm~\ref{algo:coea-process}
corresponds to lines \ref{alg:pdcoea:evaluate1}--\ref{alg:pdcoea:evaluate3}
in Algorithm~\ref{alg:pdcoea}.

\begin{algorithm}
  \caption{Pairwise Dominance CoEA (\pdcoea)\label{alg:pdcoea}}
  \begin{algorithmic}[1]
    \Require Min-max-objective function $g:\{0,1\}^n\times\{0,1\}^n\rightarrow\mathbb{R}$.
    \Require Population size $\lambda\in\mathbb{N}$ and mutation rate $\chi\in (0,n]$
    \For{$i\in [\lambda]$}\label{alg:pdcoea:init1}
    \State Sample $P_0(i)\sim\unif(\{0,1\}^n)$
    \State Sample $Q_0(i)\sim\unif(\{0,1\}^n)$\label{alg:pdcoea:init2}
    \EndFor\label{alg:pdcoea:init3}
    \For{$t\in\mathbb{N}$ until termination criterion met}
    \For{$i\in[\lambda]$}\label{alg:pdcoea:evaluate1}
    \State Sample $(x_1,y_1)\sim \unif(P_t\times Q_t)$\label{label:alg:pdcoea:selection1}
    \State Sample $(x_2,y_2)\sim \unif(P_t\times Q_t)$\label{label:alg:pdcoea:selection2}
    \If{$(x_1,y_1)\succeq_g (x_2,y_2)$}
    \State $(x,y):= (x_1,y_1)$
    \Else
    \State $(x,y) := (x_2,y_2)$\label{label:alg:pdcoea:selection3}
    \EndIf
    \State Obtain $x'$ by flipping each bit in $x$ with probability $\chi/n$.\label{label:alg:pdcoea:mutation1}
    \State Obtain $y'$ by flipping each bit in $y$ with probability $\chi/n$.\label{label:alg:pdcoea:mutation2}    
    \State Set $P_{t+1}(i):=x'$ and $Q_{t+1}(i):=y'$.
    \EndFor\label{alg:pdcoea:evaluate3}
    \EndFor
    \end{algorithmic}
\end{algorithm}

\begin{figure}
  \centering
  \includegraphics[width=6.5cm]{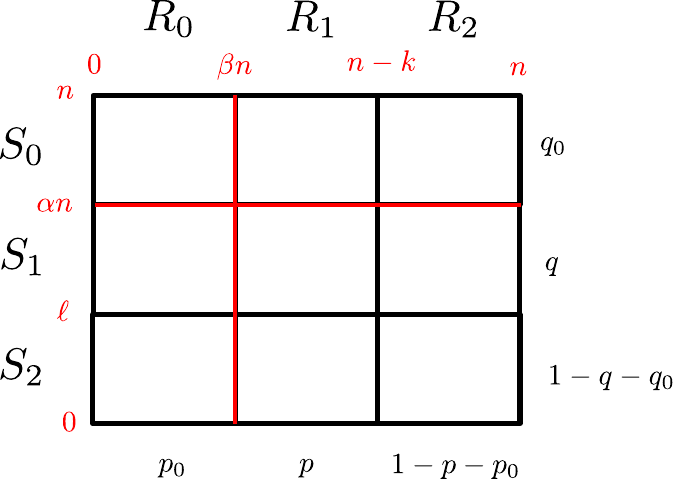}
  \caption{Partitioning of search space $\X\times\Y$ of \bilinear.}
  \label{fig:level-partition}
\end{figure}

Next, we will analyse the runtime of \pdcoea on \bilinear using
Theorem~\ref{thm:level-based}. For an arbitrary $\varepsilon\geq 1/n$
(not necessarily constant), we will restrict the analysis to the case
where $\alpha - \varepsilon> 4/5$, and $\beta < \varepsilon$. Our goal
is to estimate the time until the algorithm reaches within an
$\varepsilon$-factor of the maximin-optimal point
$(\beta n,\alpha n)$. We note that our analysis does not extend to the
general case of arbitrary $\alpha$ and $\beta$, or $\varepsilon=0$
(exact optimisation). This is a limitation of the analysis, and not of
the algorithm. Our own empirical investigations show that \pdcoea with
appropriate parameters finds the exact maximin-optimal point of
\bilinear for any value of $\alpha$ and $\beta$.

In this setting, the behaviour of the algorithm can be described
intuitively as follows. The population dynamics will have two distinct
phases. In Phase~1, most prey have less than $\alpha n$
1-bits, while most predators have more than $\beta n$
1-bits. During this phase, predators and prey will decrease the number
of 1-bits. In Phase~2, a sufficient number of predators have less than
$\beta n$ 1-bits, and the number of 1-bits in the prey-population will
start to increase. The population will then reach the
$\varepsilon$-approximation described above.

From this intuition, we will now define a suitable
sequence of levels. We will start by dividing the space $\X\times\Y$
into different regions, as shown in Figure
\ref{fig:level-partition}. Again, the $x$-axis corresponds to the
number of 1-bits in the predator, while the $y$-axis corresponds to
the number of 1-bits in the prey.

For any $k\in[0,(1-\beta)n]$, we partition \X into three sets
\begin{align}
  R_0       & := \left\{x\in\X\mid 0 \leq \|x\| < \beta n\right\}\label{eq:S0-def}\\
  R_1(k)    & := \left\{x\in\X\mid \beta n \leq \|x\| < n - k \right\}, \text{ and }\label{eq:S1-def}\\
  R_2(k)    & := \left\{x\in\X\mid n-k \leq \|x\| \leq n \right\}.\label{eq:S2-def}
\end{align}
Similarly, for any $\ell\in[0,\alpha n)$, we partition \Y into three sets
\begin{align}
  S_0       & := \left\{y\in\Y\mid \alpha n \leq \|y\|\leq n\right\}\label{eq:R0-def}\\
  S_1(\ell) & := \left\{y\in\Y\mid \ell \leq \|y\| < \alpha
              n \right\}, \text{ and}\label{eq:R1-def}\\
  S_2(\ell) & := \left\{y\in\Y\mid  0 \leq \|y\| < \ell \right\}.\label{eq:R2-def}
\end{align}
For ease of notation, when the parameters $k$ and $\ell$ are clear 
from the context, we will simply refer to these sets as $R_0,R_1,R_2,
S_0,S_1$, and $S_2$. 
Given two populations $P$ and $Q$, and $C\subseteq \X\times \Y$, define
\begin{align*}
  \Punif{C} & := \Pr_{(x,y)\sim\unif(P\times Q)}\left((x,y)\in C \right)\\
  \Psel{C} & := \Pr_{(x,y)\sim\select(P\times Q)}\left((x,y)\in C \right).
\end{align*}
In the context of subsets of $\X\times \Y$, the set $R_i$ refers to
$R_i\times \Y$, and $S_i$ refers to $\X\times S_i$.  With the above
definitions, we will introduce the following quantities which depend
on $k$ and $\ell$:
\begin{align*}
  p_0 & := \Punif{R_0} &   p(k)   & := \Punif{R_1(k)} &
  q_0 & := \Punif{S_0} &   q(\ell)   & := \Punif{S_1(\ell)}
\end{align*}

During Phase~1, the typical behaviour is that only a small minority of
the individuals in the $Q$-population belong to region $S_0$. In this
phase, the algorithm ``progresses'' by decreasing the number of 1-bits
in the $P$-population. In this phase, the number of 1-bits will
decrease in the $Q$-population, however it will not be necessary to
analyse this in detail. To capture this, we define the levels for
Phase~1 for $j\in[0..(1-\beta)n]$ as
$A^{(1)}_j := R_0\cup R_1(j)$ and
$B^{(1)}_j  := S_2((\alpha-\varepsilon) n).$

During Phase~2, the typical behaviour is that there is a sufficiently
large number of $P$-individuals in region $R_0$, and the algorithm
progresses by increasing the number of 1-bits in the
$Q$-population. The number of 1-bits in the $P$-population will
decrease or stay at 0. To capture this, we define the levels for
Phase~2 for $j\in[0,(\alpha -\varepsilon)n]$
$A^{(2)}_j := R_0$ and
$B^{(2)}_j := S_1(j).$

The overall sequence of levels used for Theorem~\ref{thm:level-based}
becomes 
\[
  (A_0^{(1)}\times B_0^{(1)}), \ldots, (A^{(1)}_{(1-\beta)n},B^{(1)}_{(1-\beta)n}),
  (A_0^{(2)}\times B_0^{(2)}), \ldots, (A^{(2)}_{(\alpha-\varepsilon)n},B^{(2)}_{(\alpha-\varepsilon)n}),
\]

The notion of ``current level'' from Theorem~\ref{thm:level-based}
together with the level-structure can be exploited to infer properties
about the populations, as the following lemma demonstrates.

\begin{lemma}\label{lemma:phase1-condition}
  If the current level is $A^{(1)}_j\times B^{(1)}_j$, then $p_0<\gamma_0/(1-q_0)$.
\end{lemma}
\begin{proof}
  Assume by contradiction that $p_0(1-q_0)\geq \gamma_0$. Note that by (\ref{eq:R2-def}),
  it holds $S_2(0) = \emptyset.$ Therefore, $1-q(0)-q_0=0$ and $q(0)=1-q_0$.
  By the definitions of the levels in Phase~2 and (\ref{eq:R1-def}),
  \begin{align*}
    \left\vert(P\times Q)\cap (A^{(2)}_0\times B_0^{(2)})\right\vert & = 
    \left\vert(P\times Q)\cap (R_0\times S_1(0))\right\vert\\
           & = p_0 q(0) \lambda^2
            = p_0 (1-q_0)\lambda^2
            \geq \gamma_0\lambda^2,
  \end{align*}
  implying that the current level must be 
  level $A^{(2)}_0\times B_0^{(2)}$ or a higher level in Phase~2,
  contradicting the assumption of the lemma.
\end{proof}

\subsection{Ensuring Condition (G2) during Phase~1}

The purpose of this section is to provide the building blocks
necessary to establish conditions (G2a) and (G2b) during Phase~1.  The
progress of the population during this phase will be jeopardised if
there are too many $Q$-individuals in $S_0$. We will employ the
negative drift theorem for populations \cite{LehreNegativeDrift2010}
to prove that it is unlikely that $Q$-individuals will drift via
region $S_1$ to region $S_0$. This theorem applies to algorithms that
can be described on the form of Algorithm~\ref{alg:psva} which makes
few assumptions about the selection step. The
$Q$-population in Algorithm~\ref{alg:pdcoea} is a special case of
Algorithm~\ref{alg:psva}.

\begin{algorithm}
  \caption{Population Selection-Variation Algorithm \cite{LehreNegativeDrift2010}  \label{alg:psva}
}
  \begin{algorithmic}[1]
    \Require Finite state space $\Y$.
    \Require Transition matrix $\pmut$ over $\Y$.
    \Require Population size $\lambda\in\mathbb{N}$.
    \Require Initial population $Q_0\in\Y^\lambda$.
    \For{$t=0,1,2,\dots$ until the termination condition is met}
    \For{$i=1$ to $\lambda$}    
    \State Choose $I_t(i)\in[\lambda]$, and set $x := Q_t(I_t(i))$.
    \State Sample $x'\sim \pmut(x)$ and set $Q_{t+1}(i) := x'$.
    \EndFor
    \EndFor
  \end{algorithmic}
\end{algorithm}

We now state the negative drift theorem for populations.

\begin{theorem}[\cite{LehreNegativeDrift2010}]\label{thm:neg-drift-pop}
  Given Algorithm~\ref{alg:psva} on $\Y=\{0,1\}^n$ with
  population size $\lambda\in \poly(n)$, and 
  transition matrix \pmut corresponding to flipping each bit independently with
  probability $\chi/n$.
  Let $a(n)$
  and $b(n)$ be positive integers s.t. $b(n)\leq n/\chi$ and
  $d(n):=b(n)-a(n)=\omega(\ln n)$. For an
  $x^*\in\{0,1\}^n$, let $T(n)$ be the smallest
  $t\geq 0$, s.t. $\min_{j\in[\lambda]}H(P_t(j), x^*)\leq a(n)$.
  Let $S_t(i) := \sum_{j=1}^\lambda [I_t(j)=i]$.
  If there are constants $\alpha_0\geq 1$ and $\delta>0$ such that
  \begin{description}
   \item[1)] $\expect{S_t(i)\mid a(n)<H(P_t(i),x^*)<b(n)}\leq \alpha_0$ for all $i\in[\lambda]$
   \item[2)] $\psi := \ln(\alpha_0)/\chi + \delta< 1$, and
   \item[3)] $\frac{b(n)}{n} < \min\left\{ \frac{1}{5}, \frac{1}{2}-\frac{1}{2}\sqrt{\psi(2-\psi)}\right\}$,
  \end{description}
  then $\prob{T(n)\leq e^{cd(n)}}\leq e^{-\Omega(d(n))}$ for some constant $c>0$.
\end{theorem}

To apply this theorem, the first step is to estimate the reproductive
rate \cite{LehreNegativeDrift2010} of $Q$-individuals in
$S_0\cup S_1$.

\begin{lemma}\label{lemma:r0r-decrease}
  If there exist $\delta_1,\delta_2\in(0,1)$ such that
  $q+q_0\leq 1-\delta_1$, $p_0<\sqrt{2(1-\delta_2)}-1$, and $p_0q=0$,
  then 
  $\Psel{S_0\cup S_1}/\Punif{S_0\cup S_1} < 1-\delta_1\delta_2.
  $\end{lemma}
\begin{proof}The conditions of Lemma~\ref{lemma:b3-growth} are satisfied, hence
  \begin{align*}
    \Psel{S_0\cup S_1} & = 1-\Psel{S_2}\\
                       & \leq 1 - (1+\delta_2(q_0+q))\Punif{S_2}\\
                       & =    1 - (1+\delta_2(q_0+q))(1-q-q_0)\\
                       & = (q_0+q)(1-(1-q-q_0)\delta_2)\\
                       & \leq (q_0+q)(1-\delta_1\delta_2)\\
                       & = \Punif{S_0\cup S_1}(1-\delta_1\delta_2).
  \end{align*}
\end{proof}

\begin{lemma}\label{lemma:r0-reproductive-rate}
  If $p_0=0$ and $q_0+q\leq 1/3$, then no $Q$-individual in $Q\cap
  (S_0\cup S_1)$ has reproductive rate higher than $1$.
\end{lemma}
\begin{proof}
  Consider any individual $z\in Q\cap (S_0\cup S_1)$. The probability
  of selecting this individual in a given iteration is less than
  \begin{multline*}
    \prob{y_1=z\wedge y_2\in S_0\cup S_1}\prob{(x_1,y_1)\succeq (x_2,y_2)\mid y_1=z\wedge y_2\in S_0\cup S_1}\\
    + \prob{y_2=z\wedge y_1\in S_0\cup S_1}\prob{(x_1,y_1)\not\succeq (x_2,y_2)\mid y_2=z\wedge y_1\in S_0\cup S_1}\\
      + \prob{y_2=z\wedge y_1\in S_2}\prob{(x_1,y_1)\not\succeq (x_2,y_2)\mid y_2=z\wedge y_1\in S_2}\\
      \leq \frac{2}{\lambda}(q_0+q)+(1-q-q_0)/(2\lambda)
      = \frac{1}{2\lambda}\left(1+3(q+q_0)\right) \leq \frac{1}{\lambda}.
    \end{multline*}
    Hence, within one generation of $\lambda$ iterations, the expected
    number of times this individual is selected is at most 1.
\end{proof}

We now have the necessary ingredients to prove the required condition
about the number of $Q$-individuals in $S_0$.

\begin{lemma}\label{lemma:no-r0}
  Assume that
  $\lambda\in\poly(n),$ and for two constants
  $\alpha,\varepsilon\in(0,1)$ with
  $\alpha-\varepsilon \geq 4/5$, the mutation rate is
  $\chi \leq 1/(1-\alpha+\varepsilon).$
  Let $T$ be as defined in Theorem~\ref{thm:main-result}.
  For any $\tau\leq e^{cn}$ where $c$ is a sufficiently small constant,
  define $\tau_*:=\min \{ T/\lambda-1, \tau\}$, then
  \begin{align*}
    \prob{\bigvee_{t=0}^{\tau_*} (Q_t\cap S_0) \neq \emptyset}
    \leq \tau e^{-\Omega(n)}+\tau e^{-\Omega(\lambda)}.
  \end{align*}
\end{lemma}
\begin{proof}
  Each individual in the initial population $Q_0$ is sampled uniformly
  at random, with $n/2 \leq (\alpha-\varepsilon) n/(1+3/5)$
  expected number of 1-bits. Hence, by a Chernoff bound
  \cite{motwani:randomized} and a union
  bound, the probability that the initial population $Q_0$ intersects
  with $S_0\cup S_1$ is no more than
  $\lambda e^{-\Omega(n)}=e^{-\Omega(n)}$. 

  We divide the remaining $t-1$ generations into a random number of
  phases, where each phase lasts until $p_0>0$, and we assume that the
  phase begins with $q_0=0$.

  If a phase begins with $p_0>0$, then the phase lasts one
  generation. Furthermore, it must hold that
  $q((\alpha-\varepsilon)n)=0$, otherwise the product $P_t\times Q_t$
  contains a pair in $R_0\times S_1((\alpha-\varepsilon)n)$, i.e., an
  $\varepsilon$-approximate solution has been found, which contradicts
  that $t<T/\lambda$. If $q((\alpha-\varepsilon)n)=0$, then all
  $Q$-individuals belong to region $S_2$. In order to obtain any
  $Q$-individual in region $S_0$, it is necessary that at least one of
  $\lambda$ individuals mutates at least
  $\varepsilon n$ 0-bits, an event which holds with probability at
  most
  $\lambda\cdot {n\choose \varepsilon n}\left(\frac{\chi}{n}\right)^{\varepsilon n} \leq
    \lambda e^{-\Omega(n)} = e^{-\Omega(n)}.
  $

  If a phase begins with $p_0=0$, then we will apply
  Theorem~\ref{thm:neg-drift-pop} to show that it is unlikely that any
  $Q$-individual will reach $S_0$ within $e^{cn}$ 
  generations, or the phase ends. We use the parameter
  $x^*:=1^n$,  $a(n):=(1-\alpha)n$, and
  $b(n):=(1-\alpha+\varepsilon)n<n/\chi$. Hence,
  $d(n):=b(n)-a(n)=\varepsilon n=\omega(\ln(n))$.
  
  We first bound the reproductive rate of $Q$-individuals in $S_1$. For any
  generation $t$, if $q_0+q<(1-\delta_2)$, then by
  Lemma~\ref{lemma:r0r-decrease}, and a Chernoff bound,
  $\vert Q_{t+1}\cap S_0\cup S_1\vert\leq (q_0+q)\lambda$ with probability
  $1-e^{-\Omega(\lambda)}$. By a union bound,
  this holds with probability
  $1-te^{-\Omega(\lambda)}$ within the next $t$ generations. Hence, by
  Lemma~\ref{lemma:r0-reproductive-rate}, the reproductive rate of any
  $Q$-individual within $S_0\cup S_1$ is at most $\alpha_0:=1$, and
  condition 1 of Theorem~\ref{thm:neg-drift-pop} is
  satisfied. Furthermore, $\psi := \ln(\alpha_0)/\chi+\delta = \delta'  < 1$
  for any $\delta'\in(0,1)$ and $\chi>0$, hence condition 2 is
  satisfied. Finally, condition 3 is satisfied as long as $\delta'$ 
  is chosen sufficiently small. It follows by Theorem~\ref{thm:neg-drift-pop}
  that the probability that a $Q$-individual in $S_0$ is produced
  within a phase of length at most $\tau<e^{cn}$ is $e^{-\Omega(n)}$.
  
  The lemma now follows by taking a union bound over the at most $\tau$
  phases. 
\end{proof}

We can now proceed to analyse Phase~1, assuming that $q_0=0$.
For a lower bound and to simplify calculations, we pessimistically
assume that the following event occurs with probability 0
\begin{align*}
  (x_1,y_1)\in (R_1\times (S_1\cup S_2))\;\wedge\; (x_2,y_2)\in (R_2\times S_0).
\end{align*}

We will see that the main effort in applying
Theorem~\ref{thm:level-based} is to prove that conditions (G2a) and
(G2b) are satisfied. The following lemma will be useful in this regard
for phase 1. Consistent with the assumptions of phase 1, the lemma
assumes an upper bound on the number of predators in region $R_0$ and
prey in region $S_0$, and no predator-prey pairs in region
$R_0\times S_1$.  Under these assumptions, the lemma implies that the
number of pairs in region $(R_0\cup R_1)\times S_2$ will increase.

\begin{lemma}\label{lemma:a1a2b3-growth}
  If there exist $\rho,\psi\in(0,1)$ such that
  \begin{description}
  \item[1)] $p_0\leq \sqrt{2(1-\rho)}-1$
  \item[2)] $q_0\leq \sqrt{2(1-\rho)}-1$
\item[3)] $p_0q=0$
  \end{description}
  then if $(p_0+p)(1-q-q_0) \leq \psi$, it holds that
  \begin{align*}
    \varphi:=\frac{\Psel{R_0\cup R_1}}{\Punif{R_0\cup R_1}}\cdot 
    \frac{\Psel{S_2}}{\Punif{S_2}} \geq 1+\rho(1-\sqrt{\psi}),
  \end{align*}
  otherwise, if $(p_0+p)(1-q-q_0) \geq \psi$, then 
  $\Psel{R_0\cup R_1}\Psel{S_2} \geq \psi.
  $\end{lemma}
\begin{proof}
  Given the assumptions, Lemma~\ref{lemma:a1a2-growth} and Lemma~\ref{lemma:b3-growth} imply
  \begin{align}
    \varphi \geq (1+\rho(1-p-p_0))(1+\rho(q+q_0))\geq 1.\label{eq:a1a2b3}
  \end{align}
  For the first statement, we consider two cases:

  Case 1: If $p_0+p<\sqrt{\psi}$, then by (\ref{eq:a1a2b3}) and
  $q_0+q\geq 0$, it follows
  $\varphi\geq (1+\rho(1-\sqrt{\psi}))\cdot 1.
  $

  Case 2: If $p_0+p\geq \sqrt{\psi}$, then by assumption
  $(1-q-q_0)\leq \sqrt{\psi}$. By (\ref{eq:a1a2b3}) and $1-p-p_0\geq
  0$, it follows that
  $\varphi \geq 1\cdot (1+\rho(1-\sqrt{\psi})).
  $

  For the second statement, (\ref{eq:a1a2b3}) implies
  \begin{align*}
    \Psel{R_0\cup R_1}\Psel{S_2}
    & = \varphi \Punif{R_0\cup R_1}\Punif{S_2}\\
    & = \varphi (p_0+p)(1-q-q_0)
      \geq 1\cdot \psi.
  \end{align*}
\end{proof}

\subsection{Ensuring Condition (G2) during Phase~2}
We now proceed to analyse Phase~2.

\begin{corollary}\label{corollary:a1b2-growth}
  For any $\delta_2\in(0,1),$ if $q_0\in[0,\delta_2/1200)$, $p_0q<1-\delta_2$, and $p_0\in(1/3,1]$, then
  for $\delta_2':= \min\{\delta_2/20-8q_0,1/10-12q_0,\frac{\delta_2}{300}(40-\delta_2(17-\delta_2))\}$, it holds
  \begin{align}
    \frac{\Psel{R_0}}{\Punif{R_0}}
    \frac{\Psel{S_1}}{\Punif{S_1}}
    > 1+\delta_2'.
  \end{align}  
\end{corollary}
\begin{proof}
  We distinguish between two cases with respect to the value of $p_0$.

  If $p_0\in(1/3,1-\delta_2/10)$, then we apply
  Lemma~\ref{lemma:a1b2-growth1}. The conditions of Lemma~\ref{lemma:a1b2-growth1} hold
  for the parameter $\delta_1:=\delta_2/10$, and the statement follows
  for
\begin{align*}
    \delta_2'=\delta_1' & = \min\{\delta_1/2-8q_0,1/10-12q_0\}\\
                        & = \min\{\delta_2/20-8q_0,1/10-12q_0\}.
  \end{align*}
   
  If $p_0\in[1-\delta_2/10,1]$, then we apply
  Lemma~\ref{lemma:a1b2-growth2} for $\rho=\delta_2$, which implies
  that the statement holds for
  \begin{align}
    \delta_2'=\frac{\delta_2}{300}(40-\delta_2(17-\delta_2)).
  \end{align}
\end{proof}

\subsection{Main Result}

We now obtain the main result:
Algorithm~\ref{alg:pdcoea} can efficiently locate an
$\varepsilon$-approximate solution to an instance of \bilinear.

\begin{theorem}\label{thm:main-result}
  Assume that $2000\ln(n)\leq \lambda\in\poly(n)$ and
  $\chi=\frac{1}{2}\ln\left(\frac{42}{41(1+\delta)}\right)$ for any
  constant $\delta\in(0,1/41)$. Let $\alpha,\beta,\varepsilon\in(0,1)$
  be three constants where  $\alpha-\varepsilon \geq 4/5$.
  Define 
  $T := \min \{\lambda t\mid (P_t\times Q_t) \cap (R_0\times S_1((\alpha-\varepsilon))n)\}
  $
  where $P_t$ and $Q_t$ are the populations of Algorithm~\ref{alg:pdcoea} applied
  to \bilinear$_{\alpha,\beta}$. Then for all $r\in\poly(n)$ and any
      constant $c''>1$, it holds
  \begin{align*}
    \prob{T>\frac{2rc''\lambda}{\delta}\left(\lambda^2n
                    +\frac{23n}{\chi}\ln\left(\frac{1}{\beta(1-\alpha+\varepsilon)}\right)\right)}  \leq (1/r)(1+o(1)).
  \end{align*}
\end{theorem}
\begin{proof}
  Note that $0 < \chi < 1 < 1/(1-\alpha+\varepsilon).$

  The proof will refer to four parameters
  $\rho,\delta,\delta_3,\gamma_0\in(0,1)$, 
  which will be defined later, but which we for now assume
  satisfy the following four constraints 
  \begin{align}
    1/3 < \gamma_0 &\leq \sqrt{2(1-\rho)}-1 < 1/2 \label{constr:gamma0}\\
   1+\delta &\leq (1+\rho(1-\sqrt{\gamma_0(1+\delta_3)}))e^{-2\chi}(1-o(1))\label{constr:delta}\\
   1+\delta &\leq (1+\delta_3)e^{-2\chi}(1-o(1))\label{constr:delta3}\\
1+\delta &\leq (1+1/40)e^{-2\chi}(1-o(1))\label{constr:delta2'}.               
  \end{align}
  
  For some $\tau\in\poly(n)$ to be defined later, let $\tau_*:=\min \{ T/\lambda-1, \tau \}$. We will
  condition on the event that $q_0=0$ holds for the first $\tau_*$
  generations, and consider the run a failure otherwise.  By
  Lemma~\ref{lemma:no-r0}, the probability of such a failure is no
  more than $\tau e^{-\Omega(\lambda)}+\tau e^{-\Omega(n)}=e^{-\Omega(\lambda)}+e^{-\Omega(n)}$,
  assuming that the constraint $\lambda\geq c\log(n)$ holds for a
  sufficiently large constant $c$.
  
  We apply Theorem~\ref{thm:level-based} with $m=m_1+m_2$ levels, with
  $m_1=(1-\beta)n+1$ levels during phase 1, and
  $m_2=(\alpha-\varepsilon)n+1$ levels during phase 2, where the levels
  \begin{align*}
    (A_0^{(1)}\times B_0^{(1)}), \ldots, (A^{(1)}_{(1-\beta)n},B^{(1)}_{(1-\beta)n}),
    (A_0^{(2)}\times B_0^{(2)}), \ldots, (A^{(2)}_{(\alpha-\varepsilon)n},B^{(2)}_{(\alpha-\varepsilon)n}),
  \end{align*}
  are as defined in Section~\ref{sec:co-evol-algor}. Hence, in overall
  the total number of levels is $m\leq 2(n+1)$.
  
  We now prove conditions (G1), (G2a), and (G2b) separately for
  Phase~1 and Phase~2.
  
  \underline{Phase~1}: Assume that the current level belongs to phase 1 for any
  $j\in[0,(1-\beta)n]$. To prove that condition (G2a) holds, we
  will now show that the
  conditions of Lemma~\ref{lemma:a1a2b3-growth} are satisfied for the
  parameter $\psi:=\gamma_0$. By
  Lemma~\ref{lemma:phase1-condition}, we have
  $p_0<\gamma_0/(1-q_0)=\gamma_0 \leq \sqrt{2(1-\rho)}-1$, hence condition 1 is
  satisfied.  Condition 2 is satisfied by the assumption on $q_0=0$. By the
  definition of the level, $(p_0+p)(1-q-q_0)<\gamma_0=\psi$. Finally,
  for condition 3, we pessimistically assume that $p_0q=0$, otherwise
  the algorithm has already found an $\varepsilon$-approximate
  solution to the problem. All three conditions of
  Lemma~\ref{lemma:a1a2b3-growth} are satisfied. To produce an
  individual in $A_{j+1}^{(1)}$, it suffices to select and individual in
  $A_{j+1}^{(1)}$ and not mutate any of the bits, and analogously to
  produce an individual in $B_{j+1}^{(1)}$. In overall, for a sample
  $(x,y)\sim\mathcal{D}(P,Q)$, this gives
  \begin{align}
    & \prob{x\in A_{j+1}^{(1)}}\prob{y\in B_{j+1}^{(1)}} \\
    & \geq \Psel{A_{j+1}^{(1)}}\Psel{B_{j+1}^{(1)}}\left(1-\frac{\chi}{n}\right)^{2n}\\
    & \geq (1+\rho(1-\sqrt{\gamma_0}))\Punif{A_{j+1}^{(1)}}\Punif{B_{j+1}^{(1)}}e^{-2\chi}(1-o(1))\\
    & > (1+\rho(1-\sqrt{\gamma_0(1+\delta_3)}))\Punif{A_{j+1}^{(1)}}\Punif{B_{j+1}^{(1)}}e^{-2\chi}(1-o(1))\\
    & \geq \gamma(1+\delta)\label{eq:phase1-g2-a},
  \end{align}
  where the last inequality follows from assumption
  (\ref{constr:delta}).
  Condition (G2a) of the level-based theorem is therefore satisfied for Phase~1.

  We now prove condition (G2b). Assume that $\gamma_0\leq\Psel{A_{j}^{(1)}}\Psel{B_{j}^{(1)}}$.
  To produce an individual in $A_{j}^{(1)}$, it suffices to select an
  individual in $A_{j}^{(1)}$ and not mutate any of the bits, and
  analogously for $B_{j}^{(1)}$.
  For a sample $(x,y)\sim\mathcal{D}(P,Q)$, we therefore have
  \begin{align}
    \prob{x\in A_{j}^{(1)}}\prob{y\in B_{j}^{(1)}}
      & \geq \Psel{A_{j}^{(1)}}\Psel{B_{j}^{(1)}}\left(1-\frac{\chi}{n}\right)^{2n}.\label{eq:g2-sample}
  \end{align}
  To lower bound the expression in (\ref{eq:g2-sample}), we apply
  Lemma~\ref{lemma:a1a2b3-growth} again, this time with parameter
  $\psi:=\gamma_0(1+\delta_3)$. We distinguish between two cases.
  
  In the case where $\gamma_0\leq \Punif{A_{j}^{(1)}}\Punif{B_{j}^{(1)}}\leq \gamma_0(1+\delta_3)$,
  the first statement of Lemma~\ref{lemma:a1a2b3-growth} gives
  \begin{align*}
    \Psel{A_{j}^{(1)}}\Psel{B_{j}^{(1)}}\left(1-\frac{\chi}{n}\right)^{2n}
    & \geq \gamma_0(1+\rho(1-\sqrt{\psi})e^{-2\chi}(1-o(1))\\
    & = \gamma_0(1+\rho(1-\sqrt{\gamma_0(1+\delta_3)})e^{-2\chi}(1-o(1))\\
    & \geq \gamma_0(1+\delta),
  \end{align*}
  where the last inequality follows from assumption (\ref{constr:delta}).  
  In the case where $\Punif{A_{j}^{(1)}}\Punif{B_{j}^{(1)}}\geq \gamma_0(1+\delta_3)$, the
  second statement of Lemma~\ref{lemma:a1a2b3-growth} gives
  \begin{align*}
    \Psel{A_{j}^{(1)}}\Psel{B_{j}^{(1)}}\left(1-\frac{\chi}{n}\right)^{2n}
    & \geq \gamma_0(1+\delta_3)e^{-2\chi}(1-o(1))\\
    & \geq \gamma_0(1+\delta),
  \end{align*}
  where the last inequality follows from assumption (\ref{constr:delta3}).
  In both cases, it follows that
  \begin{align*}
    \prob{x\in A_{j}^{(1)}}\prob{y\in B_{j}^{(1)}} \geq \gamma_0(1+\delta),
  \end{align*}
  which proves that Condition (G2b) is satisfied in phase 1.

  We now consider condition (G1).
  Assume that $\Punif{A_j^{(1)}\times B_j^{(1)}}=(p_0+p)(1-q-q_0)\geq
  \gamma_0$ and $(x,y)\sim\mathcal{D}(P,Q)$.
  Then, a $P$-individual can be obtained in $A_{j+1}^{(1)}$ by
  selecting an individual in $A_j^{(1)}$. By definition, the
  selected individual has in the worst case $n-j$ 1-bits, and it
  suffices to flip any of these bits and no other bits, an event
  which occurs with probability at least
  \begin{align*}
    \prob{x\in A_{j+1}^{(1)}} & \geq \Psel{A_j^{(1)}}\frac{(n-j)\chi}{n}\left(1-\frac{\chi}{n}\right)^{n-1}\\
       & \geq \Psel{A_j^{(1)}}\cdot \frac{(n-j)\chi}{ne^{\chi}} (1-o(1)).
  \end{align*}
  A $Q$-individual can be obtained in $B_{j+1}^{(1)}$ by selecting an
  individual in $B_{j+1}^{(1)}$ and not mutate any bits. This
  event occurs with probability at least
  \begin{align*}
    \prob{y\in B_{j+1}^{(1)}}
             & \geq \Psel{B_j^{(1)}}\left(1-\frac{\chi}{n}\right)^{n}
               \geq \frac{\gamma_0(1-o(1))}{\Psel{A_j^{(1)}}e^{\chi}}
  \end{align*}
  Hence, for a sample $(x,y)\sim\mathcal{D}(P,Q)$, we obtain by (\ref{eq:g2-sample}),
  \begin{align}
    \prob{x\in A_{j+1}^{(1)}}\prob{y\in B_{j+1}^{(1)}}
    \geq \frac{(n-j)\gamma_0\chi}{ne^{2\chi}} (1-o(1)) =:z^{(1)}_j.\label{eq:g1-bound}
\end{align}
  hence condition (G1) is satisfied.

  \underline{Phase~2}: The analysis is analogous for this
  phase. To prove (G2a), assume that the current level belongs to
  phase~2 for any $j\in[0,(\alpha-\varepsilon)n]$. By the definitions of the
  levels in this phase and the assumptions of (G2a), we must have
  \begin{align}
    p_0q(j+1)=\gamma<\gamma_0<1/2,\label{eq:constraint-for-corollary14}
  \end{align}
  and $p_0 q(j)\geq \gamma_0$, thus
  $p_0\geq \gamma_0> 1/3$ where the last inequality follows from our
  choice of $\gamma_0$. Together with the assumption $q_0=0$,
  Corollary~\ref{corollary:a1b2-growth} gives for $\delta_2:=1/2$ and
  \begin{align*}
    \delta_2' & :=\min\{\delta_2/20-8q_0,1/10-12q_0,\frac{\delta_2}{300}(40-\delta_2(17-\delta_2))\} = \frac{1}{40}.
  \end{align*}
  we get the lower bound
  \begin{align}
    \prob{x\in A_{j+1}^{(2)}}\prob{y\in B_{j+1}^{(2)}} 
    & \geq \Psel{A_{j+1}^{(2)}}\Psel{B_{j+1}^{(2)}}\left(1-\frac{\chi}{n}\right)^{2n}\\
       & \geq (1+\delta_2')\Punif{A_{j+1}^{(1)}}\Punif{B_{j+1}^{(1)}}e^{-2\chi}(1-o(1))\\
       & \geq (1+\delta)\gamma\label{eq:phase1-g2},
  \end{align}
  where the last inequality follows from assumption (\ref{constr:delta2'}).

  Condition (G2b) can be proved analogously to Phase~1. Again, we have
  \begin{align}
    \prob{x\in A_{j}^{(2)}}\prob{y\in B_{j}^{(2)}} 
    & \geq \Psel{A_{j}^{(2)}}\Psel{B_{j}^{(2)}}\left(1-\frac{\chi}{n}\right)^{2n}.
\end{align}
  In the case where
  $\Psel{A_{j}^{(2)}}\Psel{B_{j}^{(2)}}=p_0(j)q<1-\delta_2$ for $\delta_2=9/20$,
  Corollary~\ref{corollary:a1b2-growth} for
  $\delta'_2=\min(1/90,\delta_2/40)=1/90$ gives as above
  \begin{align*}
    \Psel{A_{j}^{(2)}}\Psel{B_{j}^{(2)}}\left(1-\frac{\chi}{n}\right)^{2n}
    & \geq \gamma_0(1+\delta'_2)e^{-2\chi}(1-o(1))\\
    & \geq \gamma_0(1+\delta).
  \end{align*}
  In the case where $\Psel{A_{j}^{(2)}}\Psel{B_{j}^{(2)}}=p_0(j)q\geq 1-\delta_2$, we get
  \begin{align*}
    \Psel{A_{j}^{(2)}}\Psel{B_{j}^{(2)}}\left(1-\frac{\chi}{n}\right)^{2n}
    & \geq (1-\delta_2)e^{-2\chi}(1-o(1))\\
    & = (1/2)(1+1/10)e^{-2\chi}(1-o(1))\\
    & > \gamma_0(1+\delta'_2)e^{-2\chi}(1-o(1))\\
    & \geq \gamma_0(1+\delta).
  \end{align*}
  Therefore, condition (G2b) also holds in Phase~2.
  
  To prove condition (G1), we proceed as for Phase~1 and observe that
  to produce an individual in $A_{j+1}^{(2)}$, it suffices to select
  a $P$-individual in $A_{j}^{(2)}$ and not mutate any of the
  bits. To produce an individual in $B_{j+1}^{(2)}$, it suffices to
  select a $Q$-individual in $B_{j}^{(2)}$ and flip one of the at
  least $n-j$ number of 0-bits. Similarly to in (\ref{eq:g1-bound}), we
  obtain
  \begin{align*}
    \prob{x\in A_{j+1}^{(2)}}\prob{y\in B_{j+1}^{(2)}}
& \geq \frac{(n-j)\gamma_0\chi}{ne^{2\chi}} (1-o(1)) =: z^{(2)}_j.
\end{align*}
hence condition (G1) will be satisfied during phase 2.

  Condition (G3) is satisfied as long as $\lambda\geq 2\left(\frac{1}{\gamma_0\rho^2}\right)^{1+\upsilon}\ln(m/z_*)$.
  
  All the conditions are satisfied, and assuming that $q_0=0$, it follows
  that the expected time to reach an $\varepsilon$-approximation of
  \bilinear is for any constant $c''>1$ no more than
  \begin{align}
    \expect{T} & \leq \frac{c''\lambda}{\delta}\left(\lambda^2m+16\sum_{i=1}^{m-1}\frac{1}{z_i}\right)\\
               & \leq \frac{2c''\lambda}{\delta}\left(\lambda^2(n+1)+8\sum_{i=1}^{m_1-1}\frac{1}{z^{(1)}_i}+8\sum_{i=1}^{m_2-1}\frac{1}{z^{(2)}_i}\right)\\
               & \leq
                 \frac{2c''\lambda}{\delta}\left(\lambda^2(n+1)
                 +\frac{8e^{2\chi}n(1+o(1))}{\gamma_0\chi}\left(
                 \sum_{i=1}^{m_1-1} \frac{1}{n-i} +
                 \sum_{i=1}^{m_2-1} \frac{1}{n-i}
                 \right)\right)\\
               & \leq
                 \frac{2c''\lambda}{\delta}\left(\lambda^2(n+1)
                 +\frac{8e^{2\chi}n(1+o(1))}{\gamma_0\chi}\left(
                 2\sum_{i=1}^{n-1}\frac{1}{i}
                 -\sum_{i=1}^{\beta n} \frac{1}{i}
                 -\sum_{i=1}^{(1-\alpha+\varepsilon)n} \frac{1}{i}
                 \right)\right)\\
               & \leq
                 \frac{2c''\lambda}{\delta}\left(\lambda^2(n+1)
                 +\frac{8e^{2\chi}n(1+o(1))}{\gamma_0\chi}\ln\left(\frac{1}{\beta(1-\alpha+\varepsilon)}\right)\right).\label{eq:upper-bound-parameterised}
  \end{align}
  We now choose the parameters
  $\rho,\delta,\delta_3,\gamma_0\in(0,1)$, where numerical
  maximisation of $\delta$ subject to the constraints, give
  approximate solutions $\gamma_0=9/25$, $\delta_3=1/40$,
  $\rho=47/625$, and choosing 
  \begin{align}
    \delta := \left(1+\frac{1}{41}\right)e^{-2\chi}-1\leq  \left(1+\frac{1}{40}\right)e^{-2\chi}(1-o(1))-1 \label{eq:delta-choice},
  \end{align}
  thus assumption (\ref{constr:delta2'}) is
  satisfied. Furthermore, numerical evaluation show that the
  choices of $\delta_3,\rho,$ and $\gamma_0$ give
  \begin{align*}
    \rho(1-\sqrt{\gamma_0(1+\delta_3)}) > \frac{29}{1000} > \delta_3,
  \end{align*}
  thus assumptions (\ref{constr:delta3}) and (\ref{constr:delta})
  follow from assumption (\ref{constr:delta2'}). Finally, assumption
  (\ref{constr:gamma0}) is also satisfied because
  \begin{align*}
    \frac{1}{3} <\gamma_0=\frac{9}{25} = \sqrt{2(1-\rho)}-1 < \frac{1}{2}.
  \end{align*}

  Note that condition (G3) is satisfied since for a sufficiently small
  constant $\upsilon>0$,
  \begin{align*}
    \lambda & \geq 2000\ln(n)\\
            & \geq 2\left(\frac{1}{\gamma_0\rho^2}\right)^{1+\upsilon}\ln(n^2)(1+o(1))\\
            & \geq 2\left(\frac{1}{\gamma_0\rho^2}\right)^{1+\upsilon}\ln\left(\frac{2n^2e^{2\chi}}{\gamma_0\chi}\right)\\              
            & \geq 2\left(\frac{1}{\gamma_0\rho^2}\right)^{1+\upsilon}\ln(m/z_*)
  \end{align*}
  
  Inserting these parameter choices into
  (\ref{eq:upper-bound-parameterised}) gives
  \begin{align*}
    \expect{T}
    & \leq \frac{2c''\lambda}{\delta}\left(\lambda^2n
                   +\frac{200e^{2\chi}n}{9\chi}\ln\left(\frac{1}{\beta(1-\alpha+\varepsilon)}\right)\right)(1+o(1))\\    
    & =    \frac{2c''\lambda}{\delta}\left(\lambda^2n
                   +\frac{42}{41(1+\delta)}\frac{200n}{9\chi}\ln\left(\frac{1}{\beta(1-\alpha+\varepsilon)}\right)\right)(1+o(1))\\    
    & <    \frac{2c''\lambda}{\delta}\left(\lambda^2n
                    +\frac{23n}{\chi}\ln\left(\frac{1}{\beta(1-\alpha+\varepsilon)}\right)\right)(1+o(1))
      := \tau.
  \end{align*}
  
  By Markov's inequality, the probability that a
  solution has not been obtained in $r\tau$ time is less
  than $1/r$. Hence, in overall, taking into account all failure
  events, we obtain
  \begin{align*}
    \prob{T>r\tau}\leq 1/r + e^{-\Omega(n)} + e^{-\Omega(\lambda)}\leq (1/r)(1+o(1)).
  \end{align*}
  Since the statement holds for all choices of the constant $c''>1$,
  it also holds for
  $$\tau':=\frac{2c''\lambda}{\delta}\left(\lambda^2n
    +\frac{23n}{\chi}\ln\left(\frac{1}{\beta(1-\alpha+\varepsilon)}\right)\right)$$
  (i.e., $\tau$ without the
  extra factor $1+o(1)$), giving the final result
  \begin{align*}
    \prob{T>r\tau'}\leq 1/r + e^{-\Omega(n)} + e^{-\Omega(\lambda)}\leq (1/r)(1+o(1)).
  \end{align*}
\end{proof}

\section{A Co-Evolutionary Error Threshold}\label{sec:co-evol-error}

The previous section presented a scenario where
Algorithm~\ref{alg:pdcoea} obtains an approximate solution efficiently. We now present a general
scenario where the algorithm is inefficient. In particular, we show
that there exists a critical mutation rate above which the algorithm
fails on any problem, as long as the problem does not have too many
global optima (Theorem~\ref{thm:error-threshold}). The critical mutation rate is called the ``error threshold''
of the algorithm \cite{ochoa_error_2006,LehreNegativeDrift2010}. As
far as the author is aware, this is the first time an error threshold
has been identified in co-evolution. The proof of
Theorem~\ref{thm:error-threshold} uses the so-called negative
      drift theorem for populations (Theorem~\ref{thm:neg-drift-pop})
      \cite{LehreNegativeDrift2010}.

\begin{theorem}\label{thm:error-threshold}
  There exists a constant $c>0$ such that the following holds.
  If $A$ and $B$ are subsets of $\{0,1\}^n$ with
  $\min\{ \vert A\vert, \vert B\vert\}\leq e^{cn}$, and Algorithm~\ref{alg:pdcoea}
  is executed with population size $\lambda\in\poly(n)$ and constant
  mutation rate $\chi>\ln(2)/(1-2\delta)$ for any constant
  $\delta\in(0,1/2)$, then there exists a constant $c'$ such that
  $\prob{T_{A\times B}<e^{c'n}} = e^{-\Omega(n)}$.
\end{theorem}
\begin{proof}
  Without loss of generality, assume that $\vert B\vert\leq \vert A\vert$. For a lower
  bound on $T_{A\times B}$, it suffices to compute a lower bound on
  the time until the $Q$-population contains an element in $B$.

  For any $y\in B,$ we will apply Theorem~\ref{thm:neg-drift-pop} to
  bound $T_{y}:=\min \{ t\mid H(Q_t,y)\leq 0 \}$,
  i.e., the time until the $Q$ population contains $y$. 
  Define $a(n):=0$ and $b(n):=n\min\{1/5,1/2-(1/2)\sqrt{1-\delta^2},1/\chi\}$.
  Since $\delta$ is a constant, it follows that $d(n)=b(n)-a(n)=\omega(\ln n)$.
  Furthermore, by definition, $b(n)\leq n/\chi$.
  
  We now show that condition 1 of Theorem~\ref{thm:neg-drift-pop}
  holds for $\alpha_0:=2$. For any individual $u\in\Y$, the
  probability that the individual is selected in lines
  \ref{label:alg:pdcoea:selection1}--\ref{label:alg:pdcoea:selection3}
  is at most
  $1-\prob{y_1\neq u\wedge y_2\neq u}  = 1-(1-1/\lambda)^2 = (1/\lambda)(2-1/\lambda).
  $ 
  Thus within the $\lambda$ iterations, individual $u$ is selected
  less than 2 times in expectation. This proves condition 1.

  Condition 2 is satisfied because by the assumption on the mutation
  rate, $\psi:=\ln(\alpha_0)/\chi+\delta\leq 1-\delta<1$. Finally,
  condition 3 trivially holds because $b(n)\leq n/5$ and
  $1/2-\sqrt{\psi(2-\psi)}/2
    \leq 1/2-\sqrt{1-\delta^2}/2
    \leq b(n)/n.
  $

  All conditions are satisfied, and Theorem~\ref{thm:neg-drift-pop}
  imply that for some constant $c'$, $\prob{T_{y^*}<e^{c'n}}=e^{-\Omega(n)}.$
  Taking a union bound over all elements in $B$, we get for sufficiently
  small $c$
  \begin{align*}
    \prob{T_{A\times B}<e^{c'n}} 
     &\leq \prob{T_{B\times \mathcal{Y}}<e^{c'n}}\\
     &\leq \sum_{y\in B}\prob{T_{y}<e^{c'n}}\\
     &\leq e^{cn}\cdot e^{-\Omega(n)} = e^{-\Omega(n)}.
  \end{align*}
\end{proof}

\section{Conclusion}

Co-evolutionary algorithms have gained wide-spread interest,
with a number of exciting applications.
However, their population dynamics tend to be
significantly more complex than in standard evolutionary algorithms.
A number of pathological behaviours are reported in the literature,
preventing the potential of these algorithms. There has been a
long-standing goal to develop a rigorous theory for co-evolution which
can explain when they are efficient. A major obstacle for
such a theory is to reason about the complex interactions that occur
between multiple populations.

This paper provides the first step in developing runtime analysis for
population-based, competitive co-evolutionary algorithms.  A generic
mathematical framework covering a wide range of CoEAs is presented, along with an analytical tool to derive upper
bounds on their expected runtimes.
To illustrate the approach, we define a new co-evolutionary algorithm
\pdcoea and analyse its runtime on a bilinear maximin-optimisation
problem \bilinear. For some problem instances,
the algorithm obtains a solution within
arbitrary constant approximation ratio to the optimum within
polynomial time $O(r\lambda^3n)$ with probability $1-(1/r)(1+o(1))$
for all $r\in\poly(n)$, assuming population size
$\lambda\in\Omega(\log n)\cap\poly(n)$ and sufficiently small (but
constant) mutation rate. Additionally, we present a setting where
\pdcoea is inefficient. In particular, if the mutation rate is too
high, the algorithm needs with overwhelmingly high probability
exponential time to reach any fixed solution. This constitutes a
co-evolutionary ``error threshold''.

Future work should consider broader classes of problems, as well as
other co-evolutionary algorithms. 

\section*{Acknowledgements}
Lehre was supported by a Turing AI Acceleration Fellowship
(EPSRC grant ref EP/V025562/1).

\bibliographystyle{plain}
\bibliography{bilinear}

\begin{thebibliography}{10}

\bibitem{al-dujaili_application_2019}
Abdullah Al-Dujaili, Shashank Srikant, Erik Hemberg, and Una-May O'Reilly.
\newblock On the application of {Danskin}'s theorem to derivative-free minimax
  problems.
\newblock {\em AIP Conference Proceedings}, 2070(1):020026, February 2019.
\newblock Publisher: American Institute of Physics.

\bibitem{arcuri_novel_2008}
Andrea Arcuri and Xin Yao.
\newblock A novel co-evolutionary approach to automatic software bug fixing.
\newblock In {\em 2008 {IEEE} {Congress} on {Evolutionary} {Computation}
  ({IEEE} {World} {Congress} on {Computational} {Intelligence})}, pages
  162--168, June 2008.
\newblock ISSN: 1941-0026.

\bibitem{ECHandbook}
Thomas Back, David~B. Fogel, and Zbigniew Michalewicz.
\newblock {\em Handbook of Evolutionary Computation}.
\newblock IOP Publishing Ltd., GBR, 1st edition, 1997.

\bibitem{levelbasedanalysis2018}
Dogan Corus, Duc-Cuong Dang, Anton~V. Eremeev, and Per~Kristian Lehre.
\newblock Level-{Based} {Analysis} of {Genetic} {Algorithms} and {Other}
  {Search} {Processes}.
\newblock {\em IEEE Transactions on Evolutionary Computation}, 22(5):707--719,
  October 2018.

\bibitem{corus_theory_2018}
Dogan Corus and Per~Kristian Lehre.
\newblock Theory {Driven} {Design} of {Efficient} {Genetic} {Algorithms} for a
  {Classical} {Graph} {Problem}.
\newblock In {\em Recent {Developments} in {Metaheuristics}}, Operations
  {Research}/{Computer} {Science} {Interfaces} {Series}, pages 125--140.
  Springer, Cham, 2018.

\bibitem{dang_escaping_2021}
Duc-Cuong Dang, Anton Eremeev, and Per~Kristian Lehre.
\newblock Escaping {Local} {Optima} with {Non}-{Elitist} {Evolutionary}
  {Algorithms}.
\newblock {\em Proceedings of the AAAI Conference on Artificial Intelligence},
  35(14):12275--12283, May 2021.
\newblock Number: 14.

\bibitem{dang_non-elitist_2021}
Duc-Cuong Dang, Anton Eremeev, and Per~Kristian Lehre.
\newblock Non-elitist evolutionary algorithms excel in fitness landscapes with
  sparse deceptive regions and dense valleys.
\newblock In {\em Proceedings of the {Genetic} and {Evolutionary} {Computation}
  {Conference}}, {GECCO} '21, pages 1133--1141, New York, NY, USA, June 2021.
  Association for Computing Machinery.

\bibitem{dang_runtime_2016}
Duc-Cuong Dang and Per~Kristian Lehre.
\newblock Runtime {Analysis} of {Non}-elitist {Populations}: {From} {Classical}
  {Optimisation} to {Partial} {Information}.
\newblock {\em Algorithmica}, 75(3):428--461, July 2016.

\bibitem{DoerrTheoryTools}
Benjamin {Doerr}.
\newblock {\em Theory of Evolutionary Computation}, chapter Probabilistic Tools
  for the Analysis of Randomized Optimization Heuristics, pages 1--87.
\newblock Springer, Cham., 2020.

\bibitem{RuntimeAnalysis2020Book}
Benjamin Doerr and Frank Neumann, editors.
\newblock {\em Theory of Evolutionary Computation}.
\newblock Springer, 2020.

\bibitem{droste_upper_2006}
Stefan Droste, Thomas Jansen, and Ingo Wegener.
\newblock Upper and {Lower} {Bounds} for {Randomized} {Search} {Heuristics} in
  {Black}-{Box} {Optimization}.
\newblock {\em Theory of Computing Systems}, 39(4):525--544, July 2006.

\bibitem{RLSPD}
Mario Alejandro~Hevia Fajardo, Per~Kristian Lehre, and Shishen Lin.
\newblock Runtime analysis of a co-evolutionary algorithm: Overcoming negative
  drift in maximin-optimisation.
\newblock In {\em Proceedings of the 17th ACM/SIGEVO Conference on Foundations
  of Genetic Algorithms}, FOGA '23, pages 73--83, New York, NY, USA, 2023.
  Association for Computing Machinery.

\bibitem{Fearnley2016}
John Fearnley and Rahul Savani.
\newblock Finding {Approximate} {Nash} {Equilibria} of {Bimatrix} {Games} via
  {Payoff} {Queries}.
\newblock {\em ACM Trans. on Economics and Computation}, 4(4):1--19, August
  2016.

\bibitem{Ficici2004PhD}
Sevan~G Ficici.
\newblock {\em Solution {Concepts} in {Coevolutionary} {Algorithms}}.
\newblock PhD thesis, Brandeis University, 2004.

\bibitem{LatticeBilinear}
Mario~Alejandro Hevia~Fajardo and Per~Kristian Lehre.
\newblock How fitness aggregation methods affect the performance of competitive
  coeas on bilinear problems.
\newblock In {\em Proceedings of the Genetic and Evolutionary Computation
  Conference}, GECCO '23, pages 1593--1601, New York, NY, USA, 2023.
  Association for Computing Machinery.

\bibitem{hillis_co-evolving_1990}
W.~Daniel Hillis.
\newblock Co-evolving parasites improve simulated evolution as an optimization
  procedure.
\newblock {\em Physica D: Nonlinear Phenomena}, 42(1):228--234, June 1990.

\bibitem{jansen_cooperative_2004}
Thomas Jansen and R.~Paul Wiegand.
\newblock The {Cooperative} {Coevolutionary} (1+1) {EA}.
\newblock {\em Evolutionary Computation}, 12(4):405--434, December 2004.

\bibitem{jensen_new_2004}
Mikkel~T. Jensen.
\newblock A {New} {Look} at {Solving} {Minimax} {Problems} with
  {Coevolutionary} {Genetic} {Algorithms}.
\newblock In Mauricio G.~C. Resende and Jorge~Pinho de~Sousa, editors, {\em
  Metaheuristics: {Computer} {Decision}-{Making}}, Applied {Optimization},
  pages 369--384. Springer US, Boston, MA, 2004.

\bibitem{LehreNegativeDrift2010}
Per~Kristian Lehre.
\newblock Negative {Drift} in {Populations}.
\newblock In {\em Proceedings of the 11th {International} {Conference} on
  {Parallel} {Problem} {Solving} from {Nature} ({PPSN} 2010)}, volume 6238 of
  {\em {LNCS}}, pages 244--253. Springer Berlin / Heidelberg, 2010.

\bibitem{lehre_fitness-levels_2011}
Per~Kristian Lehre.
\newblock Fitness-levels for non-elitist populations.
\newblock {\em Proceedings of the 13th annual conference on Genetic and
  evolutionary computation - GECCO '11}, page 2075, 2011.

\bibitem{lehre_bilinear_algorithmica}
Per~Kristian Lehre.
\newblock Runtime {Analysis} of {Competitive} {Co}-evolutionary {Algorithms}
  for {Maximin} {Optimisation} of a {Bilinear} {Function}.
\newblock {\em Algorithmica}, 86(7):2352--2392, July 2024.

\bibitem{PDCoEADefendIT}
Per~Kristian Lehre, Mario Hevia~Fajardo, Jamal Toutouh, Erik Hemberg, and
  Una-May O'Reilly.
\newblock Analysis of a pairwise dominance coevolutionary algorithm and
  defendit.
\newblock In {\em Proceedings of the Genetic and Evolutionary Computation
  Conference}, GECCO '23, pages 1027--1035, New York, NY, USA, 2023.
  Association for Computing Machinery.

\bibitem{lehre_improved_2017}
Per~Kristian Lehre and Phan Trung~Hai Nguyen.
\newblock Improved {Runtime} {Bounds} for the {Univariate} {Marginal}
  {Distribution} {Algorithm} via {Anti}-concentration.
\newblock In {\em Proceedings of the {Genetic} and {Evolutionary} {Computation}
  {Conference}}, {GECCO} '17, pages 1383--1390, New York, NY, USA, 2017. ACM.

\bibitem{miyagi_adaptive_2021}
Atsuhiro Miyagi, Kazuto Fukuchi, Jun Sakuma, and Youhei Akimoto.
\newblock Adaptive scenario subset selection for min-max black-box continuous
  optimization.
\newblock In {\em Proceedings of the {Genetic} and {Evolutionary} {Computation}
  {Conference}}, {GECCO} '21, pages 697--705, New York, NY, USA, June 2021.
  Association for Computing Machinery.

\bibitem{motwani:randomized}
Rajeev Motwani and Prabhakar Raghavan.
\newblock {\em Randomized Algorithms}.
\newblock Cambridge University Press, 1995.

\bibitem{ochoa_error_2006}
Gabriela Ochoa.
\newblock Error {Thresholds} in {Genetic} {Algorithms}.
\newblock {\em Evolutionary Computation}, 14(2):157--182, June 2006.

\bibitem{oreilly_adversarial_2020}
Una-May O'Reilly, Jamal Toutouh, Marcos Pertierra, Daniel~Prado Sanchez, Dennis
  Garcia, Anthony~Erb Luogo, Jonathan Kelly, and Erik Hemberg.
\newblock Adversarial genetic programming for cyber security: a rising
  application domain where {GP} matters.
\newblock {\em Genetic Programming and Evolvable Machines}, 21(1-2):219--250,
  June 2020.

\bibitem{pollack_three_2001}
Jordan~B. Pollack, Hod Lipson, Gregory Hornby, and Pablo Funes.
\newblock Three {Generations} of {Automatically} {Designed} {Robots}.
\newblock {\em Artificial Life}, 7(3):215--223, July 2001.

\bibitem{popovici2012}
Elena Popovici, Anthony Bucci, R.~Paul Wiegand, and Edwin~D. De~Jong.
\newblock Coevolutionary {Principles}.
\newblock In Grzegorz Rozenberg, Thomas B\"ack, and Joost~N. Kok, editors, {\em
  Handbook of {Natural} {Computing}}, pages 987--1033. Springer Berlin
  Heidelberg, Berlin, Heidelberg, 2012.

\bibitem{potter_cooperative_2000}
Mitchell~A. Potter and Kenneth A.~De Jong.
\newblock Cooperative {Coevolution}: {An} {Architecture} for {Evolving}
  {Coadapted} {Subcomponents}.
\newblock {\em Evolutionary Computation}, 8(1):1--29, March 2000.

\bibitem{watson_coevolutionary_2001}
Richard~A. Watson and Jordan~B. Pollack.
\newblock Coevolutionary {Dynamics} in a {Minimal} {Substrate}.
\newblock In {\em Proceedings of the 3rd {Annual} {Conference} on {Genetic} and
  {Evolutionary} {Computation}}, {GECCO}'01, pages 702--709, San Francisco, CA,
  USA, 2001. Morgan Kaufmann Publishers Inc.
\newblock event-place: San Francisco, California.

\bibitem{wegener_methods_2000}
Ingo Wegener.
\newblock Methods for the {Analysis} of {Evolutionary} {Algorithms} on
  {Pseudo}-{Boolean} {Functions}.
\newblock In Ruhul Sarker, Masoud Mohammadian, and Xin Yao, editors, {\em
  Evolutionary {Optimization}}, pages 349--369. Springer US, Boston, MA, 2002.

\end{thebibliography}

\newpage
\appendix

\section{Technical Results}

\begin{lemma}[Lemma~18 in \cite{lehre_fitness-levels_2011}]\label{lemma:mgf-bound-kappa}
  If $Z\sim\bin(\lambda,r)$ with $r\geq \alpha(1+\delta)$, then for
  any $\kappa\in(0,\delta]$, $\expect{e^{-\kappa Z}}\leq e^{-\kappa\alpha\lambda}$.
\end{lemma}

\begin{lemma}\label{lemma:product-mgf}  
  Consider any pair of independent binomial random variables
  $X\sim\bin(\lambda,p)$ and $Y\sim\bin(\lambda,q)$, where $pq\geq
  (1+\sigma)^2z,$ $p,q,z\in(0,1)$ and $\sigma>0$.
  Then $\expect{e^{-\eta XY}}\leq e^{-\eta z\lambda^2}$ for all $\eta$
  where $0<\eta\leq \frac{\sigma}{(1+\sigma)\lambda}$.
\end{lemma}
\begin{proof}
  The proof applies Lemma~\ref{lemma:mgf-bound-kappa} twice.
  
  First, we apply Lemma~\ref{lemma:mgf-bound-kappa} for the parameters
  $Z:=X$, $\alpha:=(z/q)(1+\sigma)$ and $\kappa:=\eta Y$. The
  assumptions of the lemma then imply
  $p\geq \frac{z(1+\sigma)^2}{q} = \alpha(1+\sigma)$
  and
  $\kappa \leq \frac{\sigma Y}{(1+\sigma)\lambda} \leq \sigma$,
  i.e., the conditions of Lemma~\ref{lemma:mgf-bound-kappa} are
  satisfied. This then gives
  \begin{align}
    \expect{e^{-\eta XY}\mid Y}
    & =    \expect{e^{-\kappa X}\mid Y}
     \leq e^{-\kappa\alpha\lambda}
      = \exp\left(-\frac{\eta z}{q}(1+\sigma)\lambda Y\right).\label{eq:xy-mgf-1}
  \end{align}

  Secondly, we apply Lemma~\ref{lemma:mgf-bound-kappa} for the
  parameters $Z:=Y$, $\alpha:=q/(1+\sigma)$ and
  $\kappa:=\frac{z\eta}{q}(1+\sigma)\lambda$.  We have
  $q = \alpha(1+\sigma)$, and by the assumption on $\eta$ and the fact
  that $1\geq q\geq z>0$, it follows that
  \begin{align*}
    \kappa \leq \frac{\sigma}{(1+\sigma)\lambda}\frac{z}{q}(1+\sigma)\lambda \leq \sigma.
  \end{align*}
  The conditions of Lemma~\ref{lemma:mgf-bound-kappa} are
  satisfied, giving 
  \begin{align}
    \expect{\exp\left(-\frac{\eta z}{q}(1+\sigma)\lambda Y\right)}
    & = \expect{e^{-\kappa Y}} \leq e^{-\kappa\alpha\lambda}\\
    & = \exp\left(-\frac{z\eta}{q}(1+\sigma)\lambda\frac{q}{1+\sigma}\lambda\right)
      = e^{-\eta z\lambda^2}.\label{eq:xy-mgf-2}
  \end{align}
  By (\ref{eq:xy-mgf-1}), (\ref{eq:xy-mgf-2}), and the tower property
  of the expectation,
  it follows that
  \begin{align*}
    \expect{e^{-\eta XY}}=\expect{\expect{e^{-\eta XY}}\mid Y} < e^{-\eta z\lambda^2}.
  \end{align*}
\end{proof}

\begin{lemma}\label{lemma:a1b2-growth1}
  For any $\delta_1\in(0,1)$, if $1/3<p_0<1-\delta_1$,
then for 
  $\delta_1':=\min\{\delta_1/2-8q_0,1/10-12q_0\}$, it holds
  \begin{align*}
    \varphi:=\frac{\Psel{R_0}}{\Punif{R_0}}
    \frac{\Psel{S_1}}{\Punif{S_1}}
    > 1+\delta_1'.
  \end{align*}  
\end{lemma}
\begin{proof}
  By Lemma~\ref{lemma:a1-growth} and Lemma~\ref{lemma:b2-growth}
  \begin{align*}
    \varphi
    & > \left(\frac{3}{2}(2-p_0)p_0(1-q)+q-4q_0\right)
\times\frac{1}{2}\left((1-q_0)(3+q_0)-p_0(1-q_0(2+q_0))\right)\\
    & > \frac{1}{4}\left(3(2-p_0)p_0(1-q)+2q-8q_0\right)
\times \left(3-q_0(2+q_0)-p_0+p_0q_0(2+q_0\right)\\
    & > \frac{1}{4}\left(3(2-p_0)p_0+q(2-3(2-p_0)p_0)-8q_0\right)
\times\left(3-p_0-4q_0\right)
  \end{align*}
  
  Considering the variable $q$ independently, we distinguish between
  two cases.

  Case 1: $2<3(2-p_0)p_0$. In this case, the expression is
  minimised for $q=1$, giving
  \begin{align*}
    \varphi
    & > \frac{1}{4}\left(2-8q_0\right)\left(3-p_0-4q_0\right)\\
    & > \frac{1}{4}\left(2-8q_0\right)\left(2+\delta_1-4q_0\right)\\
    & > \frac{1}{4}\left(2(2+\delta_1)-8q_0-(2+\delta_1)8q_0\right)\\
    & > 1+\delta_1/2-8q_0.
  \end{align*}

  Case 2: $2\geq3(2-p_0)p_0$. In this case, the expression is
  minimised for $q=0$, giving
  \begin{align*}
    \varphi
    & > \frac{1}{4}\left(3(2-p_0)p_0-8q_0\right)\left(3-p_0-4q_0\right)\\
    & >
      \frac{3}{4}(2-p_0)p_0(3-p_0)-\frac{q_0}{4}\left(4\cdot 3(2-p_0)p_0+8(3-p_0)\right)\\
    & > \frac{3}{4}(2-p_0)p_0(3-p_0)-12q_0
  \end{align*}
  Note that the function $f(x):=(2-x)(3-x)x$ has derivative $f'(x)<0$
  for $(5-\sqrt{7})/2<x<1$ and $f'(x)>0$ if $1/3<x<(5-\sqrt{7})/2$.
  Hence, to determine the minimum of the expression, it suffices to
  evaluate $f$ at the extremal values $x=1/3$ and $x=1$, where
  $f(1/3)=40/27$ and $f(1)=2$. Hence, in case 2, we lower
  bound $\varphi$ by 
  $\varphi  > \frac{3}{4}\cdot \frac{40}{27}-12q_0 = \frac{10}{9}-12q_0.
  $\end{proof}

\begin{lemma}\label{lemma:a1b2-growth2}
  For any $\rho\in(0,1)$, if $p_0q<1-\rho$,
  $p_0\geq 1-\rho/10$ and $q_0<\rho/90$ then 
  \begin{align*}
    \frac{\Psel{R_0}}{\Punif{R_0}}
    \frac{\Psel{S_1}}{\Punif{S_1}}
    > 1+\frac{\rho}{300}(40-\rho(17-\rho)).
  \end{align*}
\end{lemma}
\begin{proof}
  Note first that the assumptions imply
  \begin{align}
    3p_0q_0 < \rho/30.\label{eq:3p0q0-ineq}
  \end{align}
  When $p_0$ is sufficiently large, it suffices to only consider the
  cases where both $x_1$ and $x_2$ are selected in $R_0$.  More
  precisely, conditional on the event $x_1\in R_0\wedge x_2\in R_0$,
  the probability of selecting an element in $S_1$ is
  \begin{align*}
    \Psel{S_1\mid x_1\in R_0\wedge x_2\in R_0}
    & \geq \prob{y_1\in S_1\wedge y_2\in S_1}\\
    &\quad + \prob{y_1\in S_1\wedge y_2\in S_2}/2\\
    &\quad + \prob{y_1\in S_2\wedge y_2\in S_1}\\
    & =  q^2 + q(1-q-q_0)/2 + (1-q-q_0)q\\
    & = \frac{q}{2}(3-q-3q_0).
  \end{align*}
  Hence, the unconditional probability of selecting a pair in $S_1$ is
  \begin{align*}
    \Psel{S_1} & > \frac{p_0^2q}{2}\left(3-q-3q_0\right)\\
               & >
                 \frac{p_0q}{2}\left(3(1-\rho/10)-(1-\rho)-3p_0q_0\right)\\
    \intertext{using (\ref{eq:3p0q0-ineq})}
               & > \frac{p_0q}{2}\left(2+\rho-\rho(3/10)-\rho/30\right)\\
               & = p_0q\left(1+\rho/3\right).
  \end{align*}
  Using that $\Punif{S_1}=q,$ and $\Psel{R_0}\geq p_0^2$,
  we get
  \begin{align*}
    \frac{\Psel{R_0}}{\Punif{R_0}}
    \frac{\Psel{S_1}}{\Punif{S_1}}
    & \geq \frac{p_0^2}{p_0}\frac{\Psel{S_1}}{\Punif{S_1}}\\
    & > (1-\rho/10)^2\left(1+\rho/3\right)\\
    & = 1+\frac{\rho}{300}(40-\rho(17-\rho)).
  \end{align*}     
\end{proof}

\begin{lemma}\label{lemma:a1-growth}
  \begin{align*}
    \varphi := \frac{\Psel{R_0}}{\Punif{R_0}} \geq \frac{1}{2}\left((3+q_0)(1-q_0)-p_0(1-q_0(2+q_0))\right)
  \end{align*}
\end{lemma}
\begin{proof}
  Using Lemma~\ref{lemma:half-prob}, we get
  \begin{align*}    
    \Psel{R_0} 
    & = \prob{x_1\in R_0\wedge x_2\in R_0} + \\
    &\quad + \prob{x_1\in R_0 \wedge x_2\not\in R_0}\\
    &\quad\quad\times\prob{(x_1,y_1)\succeq (x_2,y_2)\mid x_1\in R_0 \wedge x_2\not\in R_0}\\
    &\quad + \prob{x_1\not\in R_0 \wedge x_2\in R_0}\\
    &\quad\quad\times(1-\prob{(x_1,y_1)\succeq (x_2,y_2)\mid x_1\not\in R_0 \wedge x_2\in R_0})\\
    & \geq \prob{x_1\in R_0\wedge x_2\in R_0} + \\
    &\quad + \prob{(x_1,y_1)\in R_0\times S_1\cup S_2\wedge (x_1,y_1)\in R_1\cup R_2\times S_1\cup S_2}/2\\
    &\quad + \prob{x_1\not\in R_0 \wedge x_2\in R_0}(1-\prob{y_1\in S_0\wedge y_2\in S_0})\\
    &\quad \geq p_0^2+p_0(1-p_0)(1-q_0)^2/2+p_0(1-p_0)(1-q_0^2)
  \end{align*}
  Recalling that $\Punif{R_0}=p_0$, we get
  \begin{align*}
    \varphi & \geq p_0+(1-p_0)(1-q_0)^2/2+(1-p_0)(1-q_0^2)\\
            & = \frac{1}{2}\left((3+q_0)(1-q_0)-p_0(1-q_0(2+q_0))\right)
  \end{align*}
\end{proof}

\begin{lemma}\label{lemma:b2-growth}
  \begin{align*}
    \varphi := \frac{\Psel{S_1}}{\Punif{S_1}} >  \frac{3}{2}(2-p_0)p_0(1-q)+q-4q_0.
  \end{align*}
\end{lemma}
\begin{proof}
  Using Lemma~\ref{lemma:half-prob}, we get
  \begin{align*}
    \Psel{S_1} 
    & = \prob{y_1\in S_1\wedge y_2\in S_1} + \\
    &\quad + \prob{y_1\in S_1 \wedge y_2\not\in S_1}\\
    &\quad\quad\times\prob{(x_1,y_1)\succeq (x_2,y_2)\mid y_1\in S_1 \wedge y_2\not\in S_1}\\
    &\quad + \prob{y_1\not\in S_1 \wedge y_2\in S_1}\\
    &\quad\quad\times(1-\prob{(x_1,y_1)\succeq (x_2,y_2)\mid y_1\not\in S_1 \wedge y_2\in S_1})\\
    & \geq \prob{y_1\in S_1\wedge y_2\in S_1} + \\
    & \quad + \prob{(x_1,y_1)\in R_0\times S_1 \wedge (x_2,y_2)\in R_0\times S_2}/2\\
    & \quad + \prob{(x_1,y_1)\in R_0\times S_1 \wedge (x_2,y_2)\in R_1\cup R_2\times S_2}\\
    & \quad + \prob{(x_1,y_1)\in R_1\times S_1 \wedge (x_2,y_2)\in R_1\times S_0}/2\\
    & \quad + \prob{(x_1,y_1)\in R_1\times S_1 \wedge (x_2,y_2)\in R_2\times S_0}\\
    & \quad + \prob{(x_1,y_1)\in R_2\times S_1 \wedge (x_2,y_2)\in R_2\times S_0}/2\\
    & \quad + \prob{y_2\in S_1}\\
    & \quad\quad \times(1-\prob{y_1\in S_1}\\
    & \quad\quad\quad-\prob{(x_1,y_1)\in R_0\times S_0 \wedge x_2\in R_0}\\
    & \quad\quad\quad-\prob{(x_1,y_1)\in R_1\times S_2 \wedge x_2\in R_1\cup R_2}\\
    & \quad\quad\quad-\prob{(x_1,y_1)\in R_2\times S_2 \wedge x_2\in R_2})\\
    & \geq q^2+qp_0(1-q-q_0)(p_0/2+1-p_0)+\\
    & \quad + qpq_0(p/2+1-p-p_0)+q(1-p-p_0)^2q_0/2\\
    & \quad + q(1-q-p_0^2q_0\\
    & \quad\quad -(1-q-q_0)(p(1-p_0)+(1-p-p_0)^2)
  \end{align*}
  Recalling that $\Punif{S_1}=q$ and noting that $(4-p_0)p_0<4$, it follows that
  \begin{align*}
    \varphi & > \frac{3}{2}(2-p_0)p_0(1-q)+q\\
            & \quad + q_0\left(\frac{3}{2}-(4-p_0)p_0\right)+p(1-q-q_0)(1-p-p_0)\\
            & > \frac{3}{2}(2-p_0)p_0(1-q)+q-4q_0.
  \end{align*}
\end{proof}

\begin{lemma}\label{lemma:a1a2-growth}
  If there exists $\rho>0$ such that
  \begin{description}
  \item[1)] $q_0\leq \sqrt{2(1-\rho)}-1$
  \end{description}
  then
  \begin{align*}
    \varphi:=\frac{\Psel{R_0\cup R_1}}{\Punif{R_0\cup R_1}}>1+\rho(1-p-p_0).
  \end{align*}
\end{lemma}
\begin{proof}
  Using Lemma~\ref{lemma:half-prob}, we get
  \begin{align*}    
    \Psel{R_0\cup R_1} 
    & = \prob{x_1\in R_0\cup R_1\wedge x_2\in R_0\cup R_1} + \\
    &\quad + \prob{x_1\in R_0\cup R_1 \wedge x_2\not\in R_0\cup R_1}\\
    &\quad\quad\times\prob{(x_1,y_1)\succeq (x_2,y_2)\mid x_1\in R_0\cup R_1 \wedge x_2\not\in R_0\cup R_1}\\
    &\quad + \prob{x_1\not\in R_0\cup R_1 \wedge x_2\in R_0\cup R_1}\\
    &\quad\quad\times(1-\prob{(x_1,y_1)\succeq (x_2,y_2)\mid x_1\not\in R_0\cup R_1 \wedge x_2\in R_0\cup R_1})\\  
    & \geq \prob{x_1\in R_0\cup R_1\wedge x_2\in R_0\cup R_1} + \\
    & \quad + \prob{(x_1,y_1)\in R_0\cup R_1\times S_0\cup S_1 \wedge
      (x_2,y_2)\in R_2\times S_0\cup S_1}/2\\
    & \quad + \prob{x_2\in R_0\cup R_1}\\
    & \quad\quad\times(1-\prob{y_2\in S_0\wedge(x_1,y_1)\in R_2\times
      S_0}-\prob{x_1\in R_0\cup R_1})\\
    & \geq (p_0+p)^2+(p_0+p)(1-q_0)^2(1-p-p_0)/2+\\
    & \quad + (p_0+p)(1-(p_0+p)-q_0^2(1-p-p_0))
  \end{align*}
  Recalling that $\Punif{R_0\cup R_1}=p_0+p$,  and the assumption of
  the lemma, it follows that
  \begin{align*}
    \varphi & \geq 1+(1-p-p_0)((1-q_0)^2/2-q_0)\\
            & = 1+(1-p-p_0)(1/2-q_0(1-q_0/2))\\
            & \geq 1+(1-p-p_0)(1/2-(1-2\rho)/2)\\
            & =    1 + \rho(1-p-p_0).
  \end{align*}
\end{proof}

\begin{lemma}\label{lemma:b3-growth}
  If there exist $\rho>0$ such that
  \begin{description}
  \item[1)] $p_0 q = 0$.
  \item[2)] $p_0<\sqrt{2(1-\rho)}-1$
  \end{description}
  then 
  \begin{align*}
    \varphi:=\frac{\Psel{S_2}}{\Punif{S_2}} \geq 1+\rho(q_0+q).
  \end{align*}
\end{lemma}
\begin{proof}
  Using Lemma~\ref{lemma:half-prob}, we get
  \begin{align*}
    \Psel{S_2}
    & = \prob{y_1\in S_2\wedge y_2\in S_2} + \\
    &\quad + \prob{y_1\in S_2 \wedge y_2\not\in S_2}\\
    &\quad\quad\times\prob{(x_1,y_1)\succeq (x_2,y_2)\mid y_1\in S_2 \wedge y_2\not\in S_2}\\
    &\quad + \prob{y_1\not\in S_2 \wedge y_2\in S_2}\\
    &\quad\quad\times(1-\prob{(x_1,y_1)\succeq (x_2,y_2)\mid y_1\not\in S_2 \wedge y_2\in S_2})\\
    & \geq \prob{y_1\in S_2\wedge y_2\in S_2} + \\
               &\quad + \prob{(x_1,y_1)\in R_1\cup R_2\times S_2 \wedge
                 (x_2,y_2)\in R_1\cup R_2\times S_0\cup S_1}/2\\
    &\quad + \prob{y_2\in S_2}\\
    &\quad\quad\times (1-\prob{(x_1,y_1)\in R_0\times S_1}\\
    &\quad\quad\quad-\prob{(x_1,y_1)\in R_0\times S_0\wedge y_2\in S_2}-\prob{y_1\in S_2})\\
    & = (1-q-q_0)^2 \\
    &\quad + (1-q-q_0)(1-p_0)^2(q_0+q)/2 \\
    &\quad + (1-q-q_0)(1-p_0q-p_0^2q_0-(1-q-q_0))
  \end{align*}
  From $\Punif{S_2}=1-q-q_0$ and the assumptions of the lemma, 
  \begin{align*}
    \varphi & \geq 1 + (1-p_0)^2(q_0+q)/2 -p_0q-p_0^2q_0\\
            & =    1 + (q_0+q)/2 - p_0q_0(1+p_0/2)\\
            & \geq 1 + (q_0+q)/2 - q_0(1/2-\rho)\\
            & \geq 1 + (q_0+q)\rho.
  \end{align*}
\end{proof}

\begin{lemma}[\cite{dang_runtime_2016}]\label{lemma:q-lower-bound}
  For $n\in\mathbb{N}$ and $x\geq 0$, we have $1-(1-x)^n\geq
  1-e^{-xn}\geq \frac{xn}{1+xn}$
\end{lemma}

\begin{theorem}[Additive drift theorem \cite{levelbasedanalysis2018}]
\label{thm:pol-drift}
\label{thm:pol-drift-lower}
Let $(Z_t)_{t\in\mathbb{N}}$ be a discrete-time stochastic process in
$[0,\infty)$ adapted to any filtration $(\filt{t})_{t\in\mathbb{N}}$.
Define $T_a := \min\{t\in\mathbb{N} \mid Z_t \leq a\}$ for any $a\geq 0$.
For some $\varepsilon>0$ and constant $0<b<\infty$, define the conditions
\begin{enumerate}
  \item[1.1)] $\expect{Z_{t+1} - Z_{t} + \varepsilon \;; t<T_{a} \mid
      \filt{t}} \leq 0$ for all $t\in\mathbb{N}$,
  \item[1.2)] $\expect{Z_{t+1} - Z_{t} + \varepsilon \;; t<T_{a} \mid
      \filt{t}} \geq 0$ for all $t\in\mathbb{N}$,
  \item[2)] $Z_t<b$ for all $t\in\mathbb{N}$, and
  \item[3)] $\expect{T_a } < \infty$.
\end{enumerate}
If 1.1), 2), and 3) hold, then $\expect{T_a \mid \filt{0}} \leq Z_0 /
\varepsilon$. \\
If 1.2), 2), and 3) hold, then $\expect{T_a \mid \filt{0}} \geq (Z_0 - a)
/ \varepsilon$.
\end{theorem}

\begin{lemma}[Corollary 1.8.3 in \cite{DoerrTheoryTools}]\label{lemma:stoch-dom-implies-expectation}
If $X\succeq Y$, then $\expect{X}\geq \expect{Y}$.
\end{lemma}

\begin{lemma}\label{lemma:sqrt-bound}
  For all $\delta\in(0,1)$ and $\delta_1\in[0,\delta)$
  \begin{align*}
    \frac{3\delta-4\delta_1}{11}<1-\sqrt{\frac{1+\delta_1}{1+\delta}}<\frac{4\delta-3\delta_1}{8}.
  \end{align*}
\end{lemma}
\begin{proof}
By taking the first two terms of the Maclaurin series of
$\sqrt{1+x}$ for any $x\in(0,1)$, we first obtain
\begin{align}
  \sqrt{1+x} & > 1+\frac{x}{2}-\frac{x^2}{8}
> 1+\frac{3x}{8},\label{eq:sqrt-bound1}
\end{align}
where the last inequality uses the assumption $x\in(0,1)$. Similarly,
taking only the first term in the Maclaurin series, we obtain 
\begin{align}
  \sqrt{1+x} < 1 + \frac{x}{2}\label{eq:sqrt-bound2}
\end{align}
Hence, using $\delta\in(0,1)$ we obtain
\begin{align*}
  \frac{3\delta-4\delta_1}{11}
  & < \frac{3\delta-4\delta_1}{8+3\delta}
    = \frac{8+3\delta-8-4\delta_1}{8+3\delta}
    = 1-\frac{1+\frac{\delta_1}{2}}{1+\frac{3\delta}{8}}
    < 1-\frac{\sqrt{1+\delta_1}}{\sqrt{1+\delta}}.
\end{align*}
where the last inequality applies (\ref{eq:sqrt-bound1}) for
$x=\delta$ and (\ref{eq:sqrt-bound2}) for $x=\delta_1$.

Analogously, the lower bound (\ref{eq:sqrt-bound2}) for $x=\delta$, the
upper bound (\ref{eq:sqrt-bound1}) for $x=\delta_1$, and the assumption
$\delta\in(0,1)$ give
\begin{align*}
  1-\frac{\sqrt{1+\delta_1}}{\sqrt{1+\delta}}
  & < 1-\frac{1+\frac{3\delta_1}{8}}{1+\frac{\delta}{2}}
    = \frac{2+\delta}{2+\delta}-\frac{2+\frac{3\delta_1}{4}}{2+\delta}
    = \frac{\delta-\frac{3\delta_1}{4}}{2+\delta}
    < \frac{4\delta-3\delta_1}{8}.
\end{align*}
\end{proof}
\newpage
\section{An Explicit Bound on the Population Size}

The requirement ``for sufficiently large $\lambda$'' in the statement
of Theorem \ref{thm:level-based} is intended to mean
$\lambda \geq C$, where $C$ is some unspecified positive constant
independent of $m$. While this statement suffices for typical
applications, it is useful to know $C$. 
The following theorem is a variant of Theorem \ref{thm:level-based} which gives an explicit
expression for $C$ (condition (G3b)) in terms of the other constants 
in the statement. Apart from the new condition (G3b) and a rewording of
(G3a), the theorem is identical to Theorem \ref{thm:level-based}.

\begin{theorem}\label{thm:level-based-precise-lambda}
Let $c''>1$ be any constant.
Given subsets $A_j\subseteq \mathcal{X}$, $B_j\subseteq\mathcal{Y}$ for
$j\in[m]$ where $A_1:=\mathcal{X}$ and $B_1:=\mathcal{Y}$, define
    $T := \min\{t\lambda \mid (P_t\times Q_t)\cap (A_{m}\times B_m)\neq \emptyset\}$, where for all
    $t\in\mathbb{N}$, $P_t\in\mathcal{X}^\lambda$ and
    $Q_t\in\mathcal{Y}^\lambda$ are the 
    populations of Algorithm~\ref{algo:coea-process} in generation $t$.
    If there
    exist $z_1,\dots,z_{m-1},\delta \in(0,1]$,
    and $\gamma_0 \in (0,1)$
    such that
    for any populations $P\in\mathcal{X}^\lambda$ and
    $Q\in\mathcal{Y}^\lambda$ with so-called ``current level'' $j:=\max\{i\in[m]\mid
    \vert(P\times Q)\cap (A_i\times B_i)\vert\geq \gamma_0\lambda^2\}$    
\begin{description}[noitemsep,leftmargin=3em]
  \item[(G1)]
  if $j\in[m-1]$ and $(x,y)\sim \mathcal{D}(P, Q)$ then
\[
    \displaystyle \prob{x\in A_{j+1}}\prob{y\in B_{j+1}} \geq z_j,
\]
\item[(G2a)]
  for all $\gamma\in(0,\gamma_0)$,
  if $j\in[m-2]$ and
      $\vert(P\times Q) \cap (A_{j+1}\times B_{j+1})\vert  \geq \gamma\lambda^2$, then
for $(x,y)\sim \mathcal{D}(P, Q)$,
      \[
        \prob{x\in A_{j+1}}\prob{y\in B_{j+1}} \geq (1+\delta)\gamma,\]
  \item[(G2b)]
    if $j\in[m-1]$ and $(x,y)\sim \mathcal{D}(P, Q)$, then
      \[
        \prob{x\in A_{j}}\prob{y\in B_{j}} \geq (1+\delta)\gamma_0,\]
  \item[(G3a)] there exists a constant $\upsilon>0$ such that the population size $\lambda\in\mathbb{N}$ satisfies
   for $z_*:=\min_{i\in[m-1]} z_i$ 
\[
\lambda \geq 2\left(\frac{1}{\gamma_0\delta^2}\right)^{1+\upsilon}\ln\left(\frac{m}{z_*}\right),
\]
\item[(G3b)] there exists a constant $\zeta\in(0,1)$ such that the population size satisfies
\[
  \lambda \geq \max\left\{\;
    \frac{30\sqrt{c''}}{\delta^2(1-\frac{1}{\sqrt{c''}})},\;
e^{(1+v)/(v\zeta)},\;
\left(\frac{(1+v)18}{\zeta v}\right)^{(1+v)/(v(1-\zeta))}\;
  \right\},
\]  
\end{description}
  then \begin{align}
\expect{T} \leq \frac{c''\lambda}{\delta}\left(m\lambda^2+ 16\sum_{i=1}^{m-1}\frac{1}{z_i}\right).
\end{align}
\end{theorem}

The proof of this theorem applies the following lemmas which provide
upper bounds on the logarithm.
\begin{lemma}\label{lemma:log-le-poly}
For all $v> 0, \zeta\in(0,1)$ and $x\geq e^{(1+v)/(\zeta v)},$  $$\ln(x)\leq \frac{(1+v)x^{\zeta v/(1+v)}}{e\zeta v}.$$
\end{lemma}
\begin{proof}
Define the function $f(x):=x^{\zeta v/(1+v)}/\ln(x)$ and note that its
derivative 
\begin{align*}
 f'(x)= \frac{\zeta v\ln (x)-1-v}{x^{1-\frac{\zeta v}{1+v}} (1+v) \ln^2(x)}
\end{align*}
is non-negative for all $x\geq e^{(1+v)/(\zeta v)}$. Thus, $f$ is
monotonically increasing for $x\geq
e^{(1+v)/(\zeta v)}$, and
\begin{align*}
  \frac{e\zeta v}{1+v} = f(e^{(1+v)/(\zeta v)}) \le f(x) = \frac{x^{\zeta v/(1+v)}}{\ln(x)}
\end{align*}
\end{proof}
\begin{lemma}\label{lemma:lambda-div-log-lambda}
  Let $\beta>0,\zeta \in(0,1)$ and $v>0$. Then for all 
  $\lambda\ge \max\left(e^{(1+v)/(\zeta v)},\left(\frac{(1+v)\beta}{v\zeta e}\right)^{(1+v)/(v(1-\zeta ))}\right)$, 
\begin{align*}
  \frac{\lambda}{\beta\ln(\lambda)} \ge \lambda^{1/(1+v)}.
\end{align*}
\end{lemma}
\begin{proof}
  By the assumption $\lambda\geq e^{(1+v)/(\zeta v)}$, we can apply
  Lemma \ref{lemma:log-le-poly} to obtain
\begin{align*}  
  \frac{\lambda}{\beta\ln(\lambda)}
  & \ge \frac{e\zeta v\lambda^{1-\zeta v/(1+v)}}{(1+v)\beta}
     = 
      \frac{v\zeta e\lambda^{\frac{v(1-\zeta )}{1+v}}}{(1+v)\beta}\cdot
      \lambda^{1/(1+v)}
      \ge \lambda^{1/(1+v)}.
\end{align*}
\end{proof}

\begin{proof}[Proof of Theorem \ref{thm:level-based-precise-lambda}]
We apply
Theorem~\ref{thm:pol-drift} (the additive drift theorem)
with respect to the parameter $a=0$ and the
process
$
  Z_t := g\left(X_{t}^{(Y_t+1)},Y_t\right),
$
where $g$ is a level-function, and
$(Y_t)_{t\in\mathbb{N}}$ and $(X^{(j)}_t)_{t\in\mathbb{N}}$ for
$j\in[m]$ are stochastic processes, which will be defined later.
$(\mathscr{F}_t)_{t\in\mathbb{N}}$ is the filtration induced by
the populations $(P_t)_{t\in\mathbb{N}}$ and $(Q_t)_{t\in\mathbb{N}}$.

We will assume \WLOG that condition (G2a) is also satisfied for
$j=m-1$, for the following reason.
Given Algorithm~\ref{algo:coea-process} with a certain mapping~$\D$, consider
Algorithm~1 with a modified mapping~$\D'(P,Q)$:
If $(P\times Q)\cap (A_{m}\times B_m)=\emptyset$, then $\D'(P,Q)=\D(P,Q)$; otherwise $\D'(P,Q)$
assigns probability mass $1$ to some pair $(x,y)$ of~$P\times Q$ that is in
$A_{m}$, \eg, to the first one among such elements.
Note that $\D'$ meets conditions~(G1), (G2a), and (G2b). Moreover, (G2a) 
hold for $j=m-1$.
For the sequence of populations $P'_0,P'_1,\dots$ and $Q'_0,Q'_1,\dots$ of Algorithm~\ref{algo:coea-process}
with mapping~$\D'$, we can put ${T' := \min\{\lambda t \mid
(P'_t\times Q'_t)\cap (A_{m}\times B_m) \neq \emptyset\}}$. Executions of the original algorithm and
the modified one before generation $T'/\lambda$ are identical. On
generation~$T'/\lambda$ both algorithms place elements of~$A_{m}$
into the populations for the first time. Thus, $T'$  and $T$ are
equal in every realisation and their expectations are equal.

  For any level $j\in[m]$ and time $t\geq 0$, let the random variable
$
    X_t^{(j)} := \vert (P_t\times Q_t) \cap (A_j\times B_j) \vert
$
  denote the number of pairs in level $A_{j}\times B_j$ at time $t$.
  As mentioned above, the current level $Y_t$ of the algorithm at time $t$
  is defined as
  \begin{align*}
    Y_t & := \max \left\{ j\in[m] \;\mid\; X_t^{(j)} \geq  \gamma_0\lambda^2 \right\}.
    \end{align*}
  Note that  $(X^{(j)}_t)_{t\in\mathbb{N}}$ and
  $(Y_t)_{t\in\mathbb{N}}$ are adapted to the filtration
  $(\mathscr{F}_t)_{t\in\mathbb{N}}$ because they are defined in terms
  of the populations $(P_t)_{t\in\mathbb{N}}$ and $(Q_t)_{t\in\mathbb{N}}$.

  When $Y_t <m$, there exists a unique $\gamma\in[0,\gamma_0)$ such that
  \begin{align}
    X_t^{(Y_t+1)} & = \vert(P_t\times Q_t)\cap (A_{Y_t+1}\times
                    B_{Y_t+1})\vert =  \gamma \lambda^2, \text{ and}\label{lambda:eq:config-3}\\
    X_t^{(Y_t)}   & = \vert(P_t\times Q_t)\cap (A_{Y_t}\times B_{Y_t})\vert \geq \gamma_0\lambda^2. \label{lambda:eq:config-2}
  \end{align}

  Finally, we define the process $(Z_t)_{t\in\mathbb{N}}$ as
  $Z_t:=0$ if $Y_t=m$, and otherwise, if $Y_t<m$, we let
  $$
    Z_t:=g\left(X_{t}^{(Y_t+1)},Y_t\right),
  $$
  where for all $k\in[\lambda^2]$, and for all $j\in[m-1]$,
  $g(k,j):=g_1(k,j)+g_2(k,j)$ and
  \begin{align*}
    g_1(k,j) &:= \frac{\eta}{1+\eta}\cdot ((m-j)\lambda^2-k)\\
    g_2(k,j) &:= \varphi\cdot\left(\frac{e^{-\eta k}}{q_{j} } + \sum^{m-1}_{i=j+1} \frac{1}{q_i}\right),
  \end{align*}
  where $\eta\in(3\delta/(11\lambda),\delta/(2\lambda))$
  and $\varphi\in(0,1)$ are parameters which will be specified later, and 
  for $j\in[m-1]$, 
  $q_j :=  \lambda z_j/(4+\lambda z_j).
  $

  Both functions have partial derivatives $\frac{\partial g_i}{\partial k}<0$ and
  $\frac{\partial g_i}{\partial j}<0$, hence they satisfy properties 1
  and 2 of Definition~\ref{def:property}. They also satisfy property 3 because
  for all $j\in[m-1]$
  \begin{align*}
    g_1(\lambda^2,j) & = \frac{\eta}{1+\eta}((m-j)\lambda^2-\lambda^2) = g_1(0,j+1)\\
    g_2(\lambda^2,j) & > \sum^{m-1}_{i=j+1} \frac{\varphi}{q_i} = g_2(0,j+1).
  \end{align*}
  Therefore $g_1$ and $g_2$ are level functions, and thus also their
  linear combination $g$ is a level function.
  
  Due to properties 1 and 2 of level functions (see Definition~\ref{def:property}),
  it holds for all $k\in [0..\lambda^2]$ and $j\in [m-1]$
  \begin{align}
    0\leq g(k,j)\leq g(0,1) & = \frac{\eta(m-1)\lambda^2}{1+\eta} +\varphi\cdot\left(\frac{1}{q_{1} } + \sum^{m-1}_{i=2} \frac{1}{q_i}\right)\\
                             & <  \frac{\eta m\lambda^2}{1+\eta} +\sum_{i=1}^{m-1} \frac{\varphi}{q_i}\\
                            & < \frac{\eta m\lambda^2}{1+\eta}+
                               \varphi\sum_{i=1}^{m-1}\frac{4+\lambda z_i}{\lambda z_i}    \label{lambda:eq:distance-bound-0}\\
\intertext{using $\eta>0$}
                            & < m\left(\eta\lambda^2+ \frac{4\varphi}{\lambda z_*}+\varphi\right)\\
                             \intertext{using $\varphi,z_*\in(0,1)$ and  $\lambda>11\eta/(3\delta)$}
                             & < \frac{m}{z_*}\left(2\eta\lambda^2+ \frac{44\eta}{3\delta}\right)
                               \intertext{condition (G3b) implies
                               $\lambda>44/3$ and  $\lambda^2>44/(3\delta)$
                               }
                             & < \frac{3\eta\lambda^2m}{z_*}.
    \label{lambda:eq:distance-bound-1}
  \end{align}
  Hence, we have $0\leq
  Z_t<g(0,1)<\infty$ for all $t\in\mathbb{N}$ which implies that
  condition 2 of the drift theorem is satisfied.

  The drift of the process at time $t$ is $\expectt{\Delta_{t+1}}$, where
  \begin{align*}
    \Delta_{t+1}  & := g\left(X_t^{(Y_t+1)},Y_t\right)-g\left(X_{t+1}^{(Y_{t+1}+1)},Y_{t+1}\right).
  \end{align*}
We bound the drift by the law of total probability as
  \begin{align}
  \expectt{\Delta_{t+1}}
         &  =  (1-\probt{Y_{t+1}<Y_t})\expectt{\Delta_{t+1} \mid Y_{t+1}\geq Y_t} \nonumber \\
         &\quad\; + \probt{Y_{t+1}<Y_t}\expectt{\Delta_{t+1} \mid Y_{t+1}< Y_t}.\label{lambda:eq:law-tot-prob}
  \end{align}
  The event $Y_{t+1}< Y_t$ holds if and only if
  $X_{t+1}^{(Y_t)}<\gamma_0\lambda^2$, which by
      Lemma~\ref{lemma:marg-prob-to-product-prob} statement~3 for
      $\gamma:=\gamma_0+1/\lambda$ and a parameter
      $\delta_1\in(0,\delta)$ to be chosen later, and
      conditions~(G2b) and (G3a), is upper bounded by
  \begin{align}
    \probt{Y_{t+1}<Y_t}
    &= \probt{X_{t+1}^{(Y_t)}<\gamma_0\lambda^2}  \\
    &= \probt{X_{t+1}^{(Y_t)}<\lambda(\gamma\lambda-1)}  \\
    & < \exp\left(-\delta_1\gamma\lambda\left(1-\sqrt{\frac{1+\delta_1}{1+\delta}}\right)\right)\\
    \intertext{by Lemma~\ref{lemma:sqrt-bound} and $\gamma<\gamma_0$}
    & < \exp\left(-\delta_1\gamma_0\lambda\left(\frac{3\delta-4\delta_1}{11}\right)\right)\\
    \intertext{to minimise the expression, we choose $\delta_1:=(3/8)\delta$}
    & = \exp\left(-\frac{9}{16}\delta^2\gamma_0\lambda\right).
\label{lambda:eq:prob-fall}
  \end{align}
Given the low probability of the event $Y_{t+1}<Y_t$, it suffices
  to use the pessimistic bound~(\ref{lambda:eq:distance-bound-1}) 
  \begin{align}
    \expectt{\Delta_{t+1} \mid Y_{t+1}<Y_t}
     & \geq -g(0,1) 
       \label{lambda:eq:drift-fall}\end{align}

  If $Y_{t+1}\geq Y_t$, we can apply Lemma~\ref{lemma:increase}
  \begin{align*}
    &\expectt{\Delta_{t+1} \mid Y_{t+1}\geq Y_t} \geq
      \expectt{g\left(X^{(Y_t+1)}_{t},Y_t\right) - g\left(X^{(Y_t+1)}_{t+1},Y_{t}\right)
               \mid Y_{t+1}\geq Y_t}.
  \end{align*}

  If $X_{t}^{(Y_t+1)}=0$, then $X_{t}^{(Y_t+1)}\leq X_{t+1}^{(Y_t+1)}$ and
  \begin{align*}
    \expectt{g_1\left(X_t^{(Y_t+1)},Y_t\right) - g_1\left(X_{t+1}^{(Y_t+1)},Y_t\right)
             \mid Y_{t+1}\geq Y_t}\geq 0,
  \end{align*}
  because the function $g_1$ satisfies property~2 in Definition~\ref{def:property}.
  Furthermore, we have the lower bound
  \begin{multline*}
    \expectt{g_2\left(X_t^{(Y_t+1)},Y_t\right)-g_2\left(X_{t+1}^{(Y_t+1)},Y_t\right)
             \mid Y_{t+1}\geq Y_t} \\
      >    \probt{X_{t+1}^{(Y_t+1)}\geq 1}
           \left(g_2\left(0,Y_t\right)-g_2\left(1,Y_t\right)\right)
      \geq \frac{\eta\varphi}{1+\eta}.
  \end{multline*}
  where the last inequality follows because 
  \begin{align*}
  \probt{X_{t+1}^{(Y_t+1)}\geq 1}
    &= \probt{(P_{t+1}\times Q_{t+1})\cap (A_{Y_t+1}\times B_{Y_t+1})\neq \emptyset}\\
    & \geq q_{Y_t},
  \end{align*}
  due to condition (G1) and Lemma~\ref{lemma:population-upgrade},
  and
  $$ g_2\left(0,Y_t\right)-g_2\left(1,Y_t\right)
  = (\varphi/q_{Y_t})(1-e^{-\eta})
  \geq \frac{\varphi\eta}{(1+\eta)q_{Y_t}}
  $$

      In the other case, where $X_t^{(Y_t+1)}=\gamma\lambda^2\geq 1$,      
  Lemma~\ref{lemma:marg-prob-to-product-prob} and condition (G2a)
  imply for $\varphi:=\delta(1-\delta')$ for an arbitrary constant $\delta'\in(0,1)$,
  \begin{multline}
    \expectt{g_1\left(X_t^{(Y_t+1)},Y_t\right)-g_1\left(X_{t+1}^{(Y_t+1)},Y_t\right)
              \mid Y_{t+1} \geq Y_{t}} \\
      = \frac{\eta}{1+\eta}\expectt{X_{t+1}^{(Y_t+1)}\mid Y_{t+1} \geq Y_{t}}-\frac{\eta}{1+\eta} X_t^{(Y_t+1)}\\
      \geq \frac{\eta}{1+\eta}(\lambda(\lambda-1)(1+\delta)\gamma-\gamma\lambda^2) > \frac{\eta}{1+\eta}\delta(1-\delta')=\frac{\eta\varphi}{1+\eta},\label{lambda:eq:def-phi}
  \end{multline}
  where the last inequality is obtained by choosing the minimal value $\gamma=1/\lambda^2$.
  For the function $g_2$, we get
  \begin{multline*}
    \expectt{g_2\left(X_t^{(Y_t+1)},Y_t\right)-g_2\left(X_{t+1}^{(Y_t+1)},Y_t\right)
             \mid Y_{t+1} \geq Y_{t}}
      = \\
    \frac{\varphi}{q_{Y_t}}
    \left( e^{-\eta X_t^{(Y_t+1)}} - \expectt{e^{-\eta X_{t+1}^{(Y_t+1)}}} \right)>0,
  \end{multline*}
  where the last inequality is due to statement 2 of Lemma~\ref{lemma:marg-prob-to-product-prob}
      for the parameter
      \begin{align*}
        \eta := \frac{1}{\lambda}\left(1-\frac{1}{\sqrt{1+\delta}}\right).
      \end{align*}
      By Lemma \ref{lemma:sqrt-bound} for $\delta_1=0$, this parameter satisfies
      \begin{align}
         \frac{3\delta}{11\lambda}<\eta < \frac{\delta}{2\lambda} < \frac{1}{\lambda}.
        \label{lambda:eq:eta-bound}
      \end{align}
  
  Taking into account all cases, we have
  \begin{align}
    \expectt{\Delta_{t+1} \mid Y_{t+1}\geq Y_t }
      \geq \frac{\eta\varphi}{1+\eta}.\label{lambda:eq:drift-forward}
  \end{align}

  We now have bounds for all the quantities in~\eqref{eq:law-tot-prob}
  with \eqref{eq:prob-fall}, \eqref{eq:drift-fall}, and \eqref{eq:drift-forward}.
  Before bounding the overall drift $\expectt{\Delta_{t+1}}$, we
  remark that the requirement on the
  population size imposed by conditions (G3a) and (G3b) and Lemma
\ref{lemma:lambda-div-log-lambda} for $\beta:=48<18e$ imply that for any constant
$\upsilon>0$ and $C:=3$, 
  \begin{align*}
    \frac{\lambda}{16C\ln\lambda} > \lambda^{\frac{1}{1+\upsilon}}> \frac{1}{\delta^2\gamma_0},
  \end{align*}
  which implies that 
\begin{align}
  C\ln\lambda < \frac{\lambda\delta^2\gamma_0}{16}. \label{lambda:eq:loglambda-bound}
\end{align}
  The overall drift is now bounded by
  \begin{align}
  \expectt{\Delta_{t+1}}
    &=       (1 - \probt{Y_{t+1}<Y_t})\expectt{\Delta_{t+1} \mid Y_{t+1}\geq Y_t} \\
    &\quad\;    + \probt{Y_{t+1}<Y_t}\expectt{\Delta_{t+1} \mid Y_{t+1}< Y_t} \\
    & \geq \frac{\eta\varphi}{1+\eta} - \exp\left(-\frac{9}{16}\delta^2\gamma_0\lambda\right)\left(\frac{3m\eta\lambda^2}{z_*}+\frac{\eta\varphi}{1+\eta}\right)\\
    & =
      \frac{\eta\varphi}{1+\eta} -
      \exp\left(-\frac{9}{16}\delta^2\gamma_0\lambda+C\ln\lambda\right)
      \left(\frac{3m\eta\lambda^{2-C}}{z_*} +\frac{\eta\varphi}{(1+\eta)\lambda^C}\right)\\
    \intertext{by (\ref{lambda:eq:loglambda-bound})}
    & >
      \frac{\eta\varphi}{1+\eta} -
      \exp\left(-\frac{1}{2}\delta^2\gamma_0\lambda\right)
      \left(\frac{3m\eta\lambda^{2-C}}{z_*} +\frac{\eta\varphi}{(1+\eta)\lambda^C}\right)\\
    \intertext{by condition (G3a)}
    & >
      \frac{\eta\varphi}{1+\eta} -(\frac{z_*}{m})      
      \left(\frac{3m\eta\lambda^{2-C}}{z_*} +\frac{\eta\varphi}{(1+\eta)\lambda^C}\right)\\
    \intertext{recalling that $C=3$}
    & =
      \frac{\eta\varphi}{1+\eta} - \frac{3\eta}{\lambda} -\frac{\eta\varphi}{(1+\eta)\lambda^3m}\\
    \intertext{by (\ref{lambda:eq:eta-bound}), 
    $\varphi=\delta(1-\delta')<1,$ $m\geq 1$, and $\eta>0$ }
    & >
      \frac{\eta\varphi}{1+\eta} - \frac{4}{\lambda^2}\\
    \intertext{using a claim that $4/\lambda^2 < 
    \rho\eta\varphi/(1+\eta)$ for some constant $\rho\in(0,1),$ which
    will be proved in (\ref{lambda:eq:claim1}) below}
    & > \frac{\eta\varphi(1-\rho) }{1+\eta}.\label{lambda:eq:overall-drift}
  \end{align}

  We now verify condition 3 of Theorem~\ref{thm:pol-drift}, \ie,
  that $T$ has
  finite expectation. Let $p_*:=\min\{(1+\delta)(1/\lambda^2), z_*\}>0$, and
  note by conditions (G1) and (G2a) that the current level increases by at
  least one with probability
  $\probt{Y_{t+1}>Y_t}\geq (p_*)^{\gamma_0\lambda}$.
  Due to the definition of the modified process $D'$, if $Y_t=m$, then $Y_{t+1}=m$.
  Hence, the probability of reaching $Y_t=m$ is lower bounded by the
  probability of the event that the current level increases in all of
  at most $m$ consecutive
  generations, \ie, $\probt{Y_{t+m} =m} \geq (p_*)^{\gamma_0\lambda m}>0$.
  It follows that $\expect{T}<\infty$.

  By Theorem~\ref{thm:pol-drift}, the upper bound on $g(0,1)$ in \eqref{eq:distance-bound-0} 
  and the lower bound on the drift in Eq. (\ref{lambda:eq:overall-drift}) and the definition of $T$,
  \begin{align*}
    \expect{T}
    &\leq \frac{\lambda(1+\eta)g(0,1)}{\eta\varphi(1-\rho)}\\
    & < \frac{\lambda(1+\eta)}{\eta\varphi(1-\rho)}\left(\frac{\eta m\lambda^2}{1+\eta}+\varphi\sum_{i=1}^{m-1}\frac{4+\lambda z_i}{\lambda z_i}\right)\\
    & < \frac{\lambda}{(1-\rho)}\left(\frac{m\lambda^2}{\varphi}+\frac{1+\eta}{\eta}\sum_{i=1}^{m-1}\left(\frac{4}{\lambda z_i}+1\right)\right)\\
    \intertext{using Eq. (\ref{lambda:eq:eta-bound}) and $\varphi:=\delta(1-\delta')$}
    & < \frac{\lambda}{(1-\rho)}\left(\frac{m\lambda^2}{\delta(1-\delta')}+\left(\frac{11\lambda}{3\delta}+1\right)\sum_{i=1}^{m-1}\left(\frac{4}{\lambda z_i}+1\right)\right)\\    
    \intertext{noting that $1<1/\delta\leq\lambda/(3\delta)$ since 
    $\lambda> 3$ by (G3b)}
    & < \frac{\lambda}{(1-\rho)}\left(\frac{m\lambda^2}{\delta(1-\delta')}+\left(\frac{4\lambda}{\delta}\right)\sum_{i=1}^{m-1}\left(\frac{4}{\lambda z_i}+1\right)\right)\\    
    & = \frac{\lambda}{(1-\rho)}\left(\frac{m\lambda^2}{\delta(1-\delta')}+\frac{4\lambda(m-1)}{\delta}+\left(\frac{16}{\delta}\right)\sum_{i=1}^{m-1}\frac{1}{z_i}\right)\\        
    & < \frac{\lambda}{(1-\rho)}\left(\frac{m\lambda^2}{\delta(1-\delta')}\left(1+\frac{4}{\lambda}\right)+\left(\frac{16}{\delta}\right)\sum_{i=1}^{m-1}\frac{1}{z_i}\right)\\        
    \intertext{using the the claim 
    $\lambda >4(1-\delta')/\delta'$ which will be shown in (\ref{lambda:eq:claim2})}
    & < \frac{\lambda}{(1-\rho)\delta}\left(\frac{m\lambda^2}{(1-\delta')^2}+16\sum_{i=1}^{m-1}\frac{1}{z_i}\right)\\    
    \intertext{choosing $\rho:=1-1/\sqrt{c''}$ and $\delta':=\rho/2$
    so that $c''=(1-\rho)^{-2}>(1-\rho)^{-1}(1-\delta')^{-2}$ give}
    & < \frac{c''\lambda}{\delta}\left(m\lambda^2+16\sum_{i=1}^{m-1}\frac{1}{z_i}\right).
  \end{align*}

It remains to prove two claims about inequalities involving the
population size $\lambda$.
First, using the upper and lower bounds in
(\ref{lambda:eq:eta-bound}), we have
   \begin{align}
     \frac{\rho\eta\varphi}{1+\eta}
     & >  \frac{\rho\eta\varphi}{2}
       > \frac{3\rho\delta\varphi}{22\lambda}
       = \frac{3\rho\delta^2(1-\delta')}{22\lambda}
       > \frac{3\rho\delta^2(1-\rho)}{22\lambda}
       = \frac{3\delta^2(1-1/\sqrt{c''})}{22\lambda\sqrt{c''}}
       >
       \frac{4}{\lambda^2},\label{lambda:eq:claim1}
   \end{align}
   where the last inequality follows from (G3b).
Secondly, 
   (G3b) and the assumptions $\delta\in(0,1]$ and $c''>1$, imply the
lower bound
   \begin{align}
     \lambda
     \geq \frac{30\sqrt{c''}}{\delta^2(1-\frac{1}{\sqrt{c''}})}
     > \frac{8}{1-\frac{1}{\sqrt{c''}}}
     = \frac{8}{\rho}
     = \frac{4}{\delta'}
     > \frac{4(1-\delta')}{\delta'}.\label{lambda:eq:claim2}
   \end{align}
\end{proof}

\end{document}